\documentclass{article}

\usepackage[preprint]{nips_2018}

\usepackage{multirow}
\usepackage[utf8]{inputenc} 
\usepackage[T1]{fontenc}    
\usepackage{hyperref}       
\usepackage{url}            
\usepackage{booktabs}       
\usepackage{amsfonts}       
\usepackage{nicefrac}       
\usepackage{microtype}      
\usepackage{graphicx}
\usepackage{subcaption}
\usepackage{braket}
\usepackage{amsmath}
\usepackage{color}
\usepackage{mathtools}
\usepackage[toc,page]{appendix}
\usepackage{amsthm}
\usepackage{float}
\usepackage{multicol}

\usepackage{enumitem}
\setlist[itemize]{leftmargin=5mm}

\usepackage{picture}
\usepackage{wrapfig}

\newcommand{\D}{{\mathbb D}}

\renewcommand{\Re}{{\mathbb R}}

\newtheorem{theorem}{Theorem}

\newtheorem{lemma}[theorem]{Lemma}

\theoremstyle{definition}
\newtheorem{definition}{Definition}[section] 
 
\DeclarePairedDelimiterX{\inp}[2]{\langle}{\rangle}{#1,#2}

\makeatletter
\newcommand{\labeltext}[2]{%
  \@bsphack
  \csname phantomsection\endcsname 
  \def\@currentlabel{#1}{\label{#2}}%
  \@esphack
}
\makeatother

\makeatletter
\newcommand\footnoteref[1]{\protected@xdef\@thefnmark{\ref{#1}}\@footnotemark}
\makeatother

\title{Hyperbolic Neural Networks}

\author{
  Octavian-Eugen Ganea$^*$ \\
  Department of Computer Science\\
  ETH Z\"urich\\
  Zurich, Switzerland\\
  \texttt{octavian.ganea@inf.ethz.ch} \\
  \And
  Gary B{\'e}cigneul\thanks{Equal contribution.}  \\
  Department of Computer Science\\
  ETH Z\"urich\\
  Zurich, Switzerland\\
  \texttt{gary.becigneul@inf.ethz.ch} \\
  \And
  Thomas Hofmann  \\
  Department of Computer Science\\
  ETH Z\"urich\\
  Zurich, Switzerland\\
  \texttt{thomas.hofmann@inf.ethz.ch} \\
}

\begin{document}

\maketitle

\begin{abstract}
Hyperbolic spaces have recently gained momentum in the context of machine learning due to their high capacity and tree-likeliness properties. However, the representational power of hyperbolic geometry is not yet on par with Euclidean geometry, mostly because of the absence of corresponding hyperbolic neural network layers. This makes it hard to use hyperbolic embeddings in downstream tasks. Here, we bridge this gap in a principled manner by combining the formalism of M\"obius gyrovector spaces with the Riemannian geometry of the Poincar\'e model of hyperbolic spaces. As a result, we derive hyperbolic versions of important deep learning tools: multinomial logistic regression, feed-forward and recurrent neural networks such as gated recurrent units. This allows to embed sequential data and perform classification in the hyperbolic space. Empirically, we show that, even if hyperbolic optimization tools are limited, hyperbolic sentence embeddings either outperform or are on par with their Euclidean variants on textual entailment and noisy-prefix recognition tasks.
\end{abstract}


\section{Introduction}
It is common in machine learning to represent data as being embedded in the Euclidean space $\mathbb{R}^n$. The main reason for such a choice is simply convenience, as this space has a vectorial structure, closed-form formulas of distance and inner-product, and is the natural generalization of our intuition-friendly, visual three-dimensional space. Moreover, embedding entities in such a continuous space allows to feed them as input to neural networks, which has led to unprecedented performance on a broad range of problems, including sentiment detection~\citep{kim2014convolutional}, machine translation~\citep{bahdanau2014neural}, textual entailment~\citep{rocktaschel2015reasoning} or knowledge base link prediction~\citep{nickel2011three,bordes2013translating}.

Despite the success of Euclidean embeddings, recent research has proven that many types of complex data (e.g. graph data) from a multitude of fields (e.g. Biology, Network Science, Computer Graphics or Computer Vision) exhibit a highly non-Euclidean latent anatomy~\citep{bronstein2017geometric}. In such cases,
the Euclidean space does not provide the most powerful or meaningful geometrical representations. For example,~\cite{de2018representation} shows that arbitrary tree structures cannot be embedded with arbitrary low distortion (i.e. almost preserving their metric) in the Euclidean space with unbounded number of dimensions, but this task becomes surprisingly easy in the hyperbolic space with only 2 dimensions where the exponential growth of distances matches the exponential growth of nodes with the tree depth. 

The adoption of neural networks and deep learning in these non-Euclidean settings has been rather limited until very recently, the main reason being the non-trivial or impossible principled generalizations of basic operations (e.g. vector addition, matrix-vector multiplication, vector translation, vector inner product) as well as, in more complex geometries, the lack of closed form expressions for basic objects (e.g. distances, geodesics, parallel transport). Thus, classic tools such as multinomial logistic regression (MLR), feed forward (FFNN) or recurrent neural networks (RNN) did not have a correspondence in these geometries.

\textit{How should one generalize deep neural models to non-Euclidean domains ?} In this paper we address this question for one of the simplest, yet useful, non-Euclidean domains: spaces of constant negative curvature, i.e. \textit{hyperbolic}. Their tree-likeness properties have been extensively studied~\citep{gromov1987hyperbolic,hamann_2017,ungar2008gyrovector} and used to visualize large taxonomies~\citep{lamping1995focus+} or to embed heterogeneous complex networks~\citep{krioukov2010hyperbolic}. In machine learning, recently, hyperbolic representations greatly outperformed Euclidean embeddings for hierarchical, taxonomic or entailment data \citep{nickel2017poincar,de2018representation,ganea2018hyperbolic}. Disjoint subtrees from the latent hierarchical structure surprisingly disentangle and cluster in the embedding space as a simple reflection of the space's negative curvature. However, appropriate deep learning tools are needed to embed feature data in this space and use it in downstream tasks. For example, implicitly hierarchical sequence data (e.g. textual entailment data, phylogenetic trees of DNA sequences or hierarchial captions of images) would benefit from suitable hyperbolic RNNs.

The \textbf{main contribution} of this paper is to bridge the gap between hyperbolic and Euclidean geometry in the context of neural networks and deep learning by generalizing in a principled manner both the basic operations as well as multinomial logistic regression (MLR), feed-forward (FFNN), simple and gated (GRU) recurrent neural networks (RNN) to the Poincar\'e model of the hyperbolic geometry. We do it by connecting the theory of gyrovector spaces and generalized M\"obius transformations introduced by  \citep{albert2008analytic,ungar2008gyrovector} with the Riemannian geometry properties of the manifold. We smoothly parametrize basic operations and objects in all spaces of constant negative curvature using a unified framework that depends only on the curvature value. Thus, we show how Euclidean and hyperbolic spaces can be continuously deformed into each other. On a series of experiments and datasets we showcase the effectiveness of our hyperbolic neural network layers compared to their "classic" Euclidean variants on textual entailment and noisy-prefix recognition tasks. We hope that this paper will open exciting future directions in the nascent field of Geometric Deep Learning.

\section{The Geometry of the Poincar\'e Ball}\label{sec:background}

\subsection{Basics of Riemannian geometry}
We briefly introduce basic concepts of differential geometry largely needed for a principled generalization of Euclidean neural networks. For more rigorous and in-depth expositions, see \cite{spivak1979comprehensive,hopper2010ricci}.

An \textit{$n$-dimensional manifold} $\mathcal{M}$ is a space that can locally be approximated by $\mathbb{R}^n$: it is a generalization to higher dimensions of the notion of a 2D surface. For $x\in\mathcal{M}$, one can define the \textit{tangent space} $T_x\mathcal{M}$ of $\mathcal{M}$ at $x$ as the first order linear approximation of $\mathcal{M}$ around $x$. A \textit{Riemannian metric} $g=(g_x)_{x\in\mathcal{M}}$ on $\mathcal{M}$ is a collection of inner-products $g_x:T_x\mathcal{M}\times T_x\mathcal{M}\to\mathbb{R}$ varying smoothly with $x$. A \textit{Riemannian manifold} $(\mathcal{M},g)$ is a manifold $\mathcal{M}$ equipped with a Riemannian metric $g$. Although a choice of a Riemannian metric $g$ on $\mathcal{M}$ only seems to define the geometry locally on $\mathcal{M}$, it induces global distances by integrating the length (of the speed vector living in the tangent space) of a shortest path between two points:
\begin{equation}
d(x,y)=\inf_\gamma\int_0^1\sqrt{g_{\gamma(t)}(\dot{\gamma}(t),\dot{\gamma}(t))}dt,
\end{equation}
where $\gamma\in\mathcal{C}^\infty([0,1],\mathcal{M})$ is such that $\gamma(0)=x$ and $\gamma(1)=y$. A smooth path $\gamma$ of minimal length between two points $x$ and $y$ is called a \textit{geodesic}, and can be seen as the generalization of a straight-line in Euclidean space. The \textit{parallel transport} $P_{x\to y}:T_x M\to T_y M$ is a linear isometry between tangent spaces which corresponds to moving tangent vectors along geodesics and defines a canonical way to connect tangent spaces. The \textit{exponential map} $\exp_x$ at $x$, when well-defined, gives a way to project back a vector $v$ of the tangent space $T_x\mathcal{M}$ at $x$, to a point $\exp_x(v)\in\mathcal{M}$ on the manifold. This map is often used to parametrize a geodesic $\gamma$ starting from $\gamma(0):=x\in\mathcal{M}$ with unit-norm direction $\dot{\gamma}(0):=v\in T_x\mathcal{M}$ as $t\mapsto\exp_x(tv)$. For \textit{geodesically complete manifolds}, such as the Poincar\'e ball considered in this work, $\exp_x$ is well-defined on the full tangent space $T_x\mathcal{M}$. Finally, a metric $\tilde{g}$ is said to be \textit{conformal} to another metric $g$ if it defines the same angles, \textit{i.e.}
\begin{equation}
\dfrac{\tilde{g}_x(u,v)}{\sqrt{\tilde{g}_x(u,u)}\sqrt{\tilde{g}_x(v,v)}}=\dfrac{g_x(u,v)}{\sqrt{g_x(u,u)}\sqrt{g_x(v,v)}},
\end{equation}
for all $x\in\mathcal{M}$, $u,v\in T_x\mathcal{M}\setminus\{\textbf{0}\}$. This is equivalent to the existence of a smooth function $\lambda:\mathcal{M}\to\mathbb{R}$, called the \textit{conformal factor}, such that $\tilde{g}_x=\lambda_x^2 g_x$ for all $x\in\mathcal{M}$.

\subsection{Hyperbolic space: the Poincar\'e ball}

The hyperbolic space has five isometric models that one can work with \citep{cannon1997hyperbolic}. Similarly as in \citep{nickel2017poincar} and \citep{ganea2018hyperbolic}, we choose to work in the Poincar\'e ball. The \textit{Poincar\'e ball model} $(\mathbb D^n, g^{\mathbb D})$ is defined by the manifold $\mathbb D^n = \{ x \in \mathbb \Re^n: \| x\| <1\}$ equipped with the following Riemannian metric:
\begin{align}\label{eq:metric_tensor}
g^\D_x = \lambda_x^2 g^E,\quad\text{where\ } \lambda_x := \frac 2 {1- \|x\|^2},
\end{align}
$g^E=\textbf{I}_n$ being the Euclidean metric tensor. Note that the hyperbolic metric tensor is conformal to the Euclidean one. The \textit{induced distance} between two points $x,y\in\D^n$ is known to be given by
\begin{equation}\label{eq:dist_D}
d_\D(x,y)=\cosh^{-1}\left(1+2\dfrac{\Vert x-y\Vert^2}{(1-\Vert x\Vert^2)(1-\Vert y\Vert^2)}\right).
\end{equation}
Since the Poincar\'e ball is conformal to Euclidean space, the \textit{angle} between two vectors $u,v\in T_x\D^n\setminus\{\textbf{0}\}$ is given by 
\begin{equation}
\cos(\angle(u,v))=\dfrac{g^\D_x(u,v)}{\sqrt{g^\D_x(u,u)}\sqrt{g^\D_x(v,v)}}=\dfrac{\langle u,v\rangle}{\Vert u\Vert\Vert v\Vert}.
\label{eq:hyp_angle_tangent_space}
\end{equation}

\subsection{Gyrovector spaces}\label{sec:gyro_sec}

In Euclidean space, natural operations inherited from the vectorial structure, such as vector addition, subtraction and scalar multiplication are often useful. The framework of \textit{gyrovector spaces} provides an elegant non-associative algebraic formalism for hyperbolic geometry just as vector spaces provide the algebraic setting for Euclidean geometry \citep{albert2008analytic,ungar2001hyperbolic,ungar2008gyrovector}. 

In particular, these operations are used in special relativity, allowing to add speed vectors belonging to the Poincar\'e ball of radius $c$ (the celerity, \textit{i.e.} the speed of light) so that they remain in the ball, hence not exceeding the speed of light.

We will make extensive use of these operations in our definitions of hyperbolic neural networks.

For $c\geq 0$, denote\footnote{We take different notations as in \cite{ungar2001hyperbolic} where the author uses $s=1/\sqrt{c}$.} by $\D_c^n:=\{x\in\mathbb{R}^n\mid c\Vert x\Vert^2<1\}$. Note that if $c=0$, then $\D_c^n=\Re^n$; if $c>0$, then $\D_c^n$ is the open ball of radius $1/\sqrt{c}$. If $c=1$ then we recover the usual ball $\D^n$.

\paragraph{M\"obius addition.} The \textit{M\"obius addition} of $x$ and $y$ in $\D_c^n$ is defined as 
\begin{equation}\label{eq:mobius_add}
x \oplus_c y := \dfrac{(1+2 c\langle x,y\rangle+c\Vert y\Vert^2)x+(1-c\Vert x\Vert^2)y}{1+2 c\langle x,y\rangle+ c^2\Vert x\Vert^2\Vert y\Vert^2}.
\end{equation}
In particular, when $c=0$, one recovers the Euclidean addition of two vectors in $\Re^n$. Note that without loss of generality, the case $c>0$ can be reduced to $c=1$. Unless stated otherwise, we will use $\oplus$ as $\oplus_1$ to simplify notations. For general $c>0$, this operation is not commutative nor associative. However, it satisfies $x\oplus_c \textbf{0}=\textbf{0}\oplus_c x=\textbf{0}$. Moreover, for any $x,y\in\D^n_c$, we have $(-x)\oplus_c x=x\oplus_c(-x)=\textbf{0}$ and $(-x)\oplus_c(x\oplus_c y)=y$ (left-cancellation law). The \textit{M\"obius substraction} is then defined by the use of the following notation: $x\ominus_c y:=x\oplus_c(- y)$. See \cite[section 2.1]{vermeer2005geometric} for a geometric interpretation of the M\"obius addition. 

\paragraph{M\"obius scalar multiplication.} For $c>0$, the \textit{M\"obius scalar multiplication} of $x\in\D_c^n\setminus\{\textbf{0}\}$ by $r\in\mathbb{R}$ is defined as 
\begin{equation}\label{eq:mobius_mult}
r\otimes_c x := (1/\sqrt{c}) \tanh(r\tanh^{-1}(\sqrt{c}\Vert x\Vert))\dfrac{x}{\Vert x\Vert},
\end{equation}
and $r\otimes_c\textbf{0}:=\textbf{0}$. Note that similarly as for the M\"obius addition, one recovers the Euclidean scalar multiplication when $c$ goes to zero: $\lim_{c\to 0} r\otimes_c x=rx$. This operation satisfies desirable properties such as $n\otimes_c x=x\oplus_c\dots\oplus_c x$ ($n$ additions), $(r+r')\otimes_c x=r\otimes_c x\oplus_c r'\otimes_c x$ (scalar distributivity\footnote{$\otimes_c$ has priority over $\oplus_c$ in the sense that $a\otimes_c b\oplus_c c:=(a\otimes_c b)\oplus_c c$ and $a\oplus_c b\otimes_c c:=a\oplus_c (b\otimes_c c)$.}), $(rr')\otimes_c x=r\otimes_c(r'\otimes_c x)$ (scalar associativity) and $\vert r\vert\otimes_c x/\Vert r\otimes_c x\Vert = x/\Vert x\Vert$ (scaling property). 

\paragraph{Distance.} If one defines the generalized hyperbolic metric tensor $g^c$ as the metric conformal to the Euclidean one, with conformal factor $\lambda_x^c:=2/(1- c\|x\|^2)$, then the induced distance function on $(\D_c^n,g^c)$ is given by\footnote{The notation $-x\oplus_c y$ should always be read as $(-x)\oplus_c y$ and not $-(x\oplus_c y)$.}
\begin{equation}\label{eq:mobius_dist}
d_c(x,y)=(2/\sqrt{c})\tanh^{-1}\left(\sqrt{c}\Vert -x\oplus_c y\Vert\right).
\end{equation}
Again, observe that $\lim_{c\to 0} d_c(x,y)=2\Vert x-y\Vert$, \textit{i.e.} we recover Euclidean geometry in the limit\footnote{The factor $2$ comes from the conformal factor $\lambda_x=2/(1-\|x\|^2)$, which is a convention setting the curvature to $-1$.}. Moreover, for $c=1$ we recover $d_\D$ of Eq.~(\ref{eq:dist_D}).

\paragraph{Hyperbolic trigonometry.} Similarly as in the Euclidean space, one can define the notions of hyperbolic angles or \textit{gyroangles} (when using the $\oplus_c$), as well as hyperbolic law of sines in the generalized Poincar\'e ball $(\D_c^n,g^c)$. We make use of these notions in our proofs. See Appendix~\ref{sec:hyp_trig}. 

\subsection{Connecting Gyrovector spaces and Riemannian geometry of the Poincar\'e ball}\label{ssec:connect_gyro_riem}

In this subsection, we present how geodesics in the Poincar\'e ball model are usually described with M\"obius operations, and push one step further the existing connection between gyrovector spaces and the Poincar\'e ball by finding new identities involving the exponential map, and parallel transport.

In particular, these findings provide us with a simpler formulation of M\"obius scalar multiplication, yielding a natural definition of matrix-vector multiplication in the Poincar\'e ball.

\paragraph{Riemannian gyroline element.} The Riemannian gyroline element is defined for an infinitesimal $dx$ as $ds:=(x+dx)\ominus_c x$, and its size is given by \citep[section 3.7]{ungar2008gyrovector}:
\begin{align}
\Vert ds\Vert = \| (x + dx) \ominus_c x\| =\|dx\|/(1 - c\|x\|^2).
\end{align}
What is remarkable is that it turns out to be identical, up to a scaling factor of $2$, to the usual line element $2\|dx\|/(1 - c\|x\|^2)$ of the Riemannian manifold $(\D_c^n,g^c)$.

\paragraph{Geodesics.} The geodesic connecting points $x,y \in \D_c^n$ is shown in \citep{albert2008analytic,ungar2008gyrovector} to be given by:
\begin{align}
\gamma_{x\rightarrow y}(t) := x \oplus_c (-x \oplus_c y) \otimes_c t, \quad  \text{with\ } \gamma_{x\to y}: \Re \rightarrow \D^n_c \text{\ s.t.\ } \gamma_{x\to y}(0) = x \text{\ and\ } \gamma_{x\to y}(1) = y.
\label{eq:geodesic_2pts}
\end{align}
Note that when $c$ goes to $0$, geodesics become straight-lines, recovering Euclidean geometry. In the remainder of this subsection, we connect the gyrospace framework with Riemannian geometry. 

\begin{lemma}
For any $x \in \D^n$ and $v \in T_x\D^n_c$ s.t. $g_x^c(v,v) = 1$, the \textbf{unit-speed geodesic} starting from $x$ with direction $v$ is given by:
\begin{align}
\gamma_{x,v}(t) = x \oplus_c \left(\tanh\left(\sqrt{c}\frac{t}{2} \right)\frac{v}{\sqrt{c}\|v\|}\right),\ \text{where\ }\gamma_{x,v}: \Re \to \D^n \text{\ s.t.\ } \gamma_{x,v}(0) = x \text{\ and\ } \dot{\gamma}_{x,v}(0) = v.
\label{eq:gyro_unitspeed_geodesic}
\end{align}
\end{lemma}
\vspace{-0.4cm}
\begin{proof} One can use Eq.~(\ref{eq:geodesic_2pts}) and reparametrize it to unit-speed using Eq.~(\ref{eq:mobius_dist}). Alternatively, direct computation and identification with the formula in \citep[Thm. 1]{ganea2018hyperbolic} would give the same result. Using Eq.~(\ref{eq:mobius_dist}) and Eq.~(\ref{eq:gyro_unitspeed_geodesic}), one can sanity-check that $d_c(\gamma(0), \gamma(t)) = t, \forall t \in [0,1]$. 
\end{proof}

\paragraph{Exponential and logarithmic maps.} The following lemma gives the closed-form derivation of exponential and logarithmic maps.
\begin{lemma}
For any point $x \in \D^n_c$, the exponential map $\exp_x^c: T_x\D^n_c \to \D^n_c$ and the logarithmic map $\log^c_x: \D^n_c \to T_x\D^n_c$ are given for $v\neq\textbf{0}$ and $y\neq x$ by:
\begin{align}
\hspace{-0.3cm}
\exp_x^c(v) = x \oplus_c \left(\tanh\left(\sqrt{c}\frac{\lambda_x^c \|v\|}{2} \right) \frac{v}{\sqrt{c}\|v\|} \right),\ \log^c_x(y) = \frac{2}{\sqrt{c}\lambda^c_x} \tanh^{-1}(\sqrt{c}\|-x\oplus_c y\|) \frac{-x\oplus_c y}{\|-x\oplus_c y\|}.
\label{eq:gyro_exp_map}
\end{align}
\end{lemma}
\vspace{-0.5cm}
\begin{proof}
Following the proof of \citep[Cor. 1.1]{ganea2018hyperbolic}, one gets $\exp_x^c(v) = \gamma_{x, \frac{v}{\lambda_x^c \|v\|}(\lambda_x^c \|v\|)}$. Using Eq.~(\ref{eq:gyro_unitspeed_geodesic}) gives the formula for $\exp_x^c$. Algebraic check of the identity $\log^c_x(\exp_x^c(v)) = v$ concludes.
\end{proof}

The above maps have more appealing forms when $x = \textbf{0}$, namely for $v\in T_\textbf{0}\D^n_c\setminus\{\textbf{0}\},\ y\in\D^n_c\setminus\{\textbf{0}\}$:

\begin{align}
\exp_\textbf{0}^c(v) = \tanh(\sqrt{c}\|v\|) \frac{v}{\sqrt{c}\|v\|},\ \log^c_\textbf{0}(y) = \tanh^{-1}(\sqrt{c}\|y\|) \frac{y}{\sqrt{c}\|y\|}.
\end{align}

Moreover, we still recover Euclidean geometry in the limit $c\to 0$, as $\lim_{c\to 0}\exp_x^c(v)=x+v$ is the Euclidean exponential map, and $\lim_{c\to 0}\log_x^c(y)=y-x$ is the Euclidean logarithmic map.

\paragraph{M\"obius scalar multiplication using exponential and logarithmic maps.}
We studied the exponential and logarithmic maps in order to gain a better understanding of the M\"obius scalar multiplication (Eq.~(\ref{eq:mobius_mult})). We found the following:
\begin{lemma}
The quantity $r\otimes x$ can actually be obtained by projecting $x$ in the tangent space at $\textbf{0}$ with the logarithmic map, multiplying this projection by the scalar $r$ in $T_\textbf{0}\D^n_c$, and then projecting it back on the manifold with the exponential map:
\begin{align}
r \otimes_c x = \exp_\textbf{0}^c(r \log_\textbf{0}^c(x)), \quad \forall r \in \Re, x \in \D^n_c.
\label{eq:scalar_otimes_exp}
\end{align}
In addition, we recover the well-known relation between geodesics connecting two points and the exponential map:
\begin{align}
\gamma_{x\rightarrow y}(t) = x \oplus_c (-x \oplus_c y) \otimes_c t = \exp_{x}^c(t \log_x^c(y)), \quad t \in [0,1].
\end{align}
\end{lemma}

This last result enables us to generalize scalar multiplication in order to define matrix-vector multiplication between Poincar\'e balls, one of the essential building blocks of hyperbolic neural networks. 

\paragraph{Parallel transport.} Finally, we connect parallel transport (from $T_\textbf{0}\D_c^n$) to gyrovector spaces with the following theorem, which we prove in appendix~\ref{sec:parallel_transport}.
\begin{theorem}\label{thm:parallel_transport}
In the manifold $(\D_c^n,g^c)$, the parallel transport w.r.t. the Levi-Civita connection of a vector $v\in T_\textbf{0}\D_c^n$ to another tangent space $T_x\D_c^n$ is given by the following isometry:
\begin{equation}
P^c_{\textbf{0}\to x}(v)=\log^c_x(x\oplus_c\exp_\textbf{0}^c(v))=\dfrac{\lambda_\textbf{0}^c}{\lambda_x^c}v.
\end{equation}
\end{theorem}

As we'll see later, this result is crucial in order to define and optimize parameters shared between different tangent spaces, such as biases in hyperbolic neural layers or parameters of hyperbolic MLR.




\section{Hyperbolic Neural Networks}

Neural networks can be seen as being made of compositions of basic operations, such as linear maps, bias translations, pointwise non-linearities and a final sigmoid or softmax layer. We first explain how to construct a softmax layer for logits lying in a Poincar\'e ball. Then, we explain how to transform a mapping between two Euclidean spaces as one between Poincar\'e balls, yielding matrix-vector multiplication and pointwise non-linearities in the Poincar\'e ball. Finally, we present possible adaptations of various recurrent neural networks to the hyperbolic domain. 


\subsection{Hyperbolic multiclass logistic regression}\label{sec:hyp_mlr}
In order to perform multi-class classification on the Poincar\'e ball, one needs to generalize multinomial logistic regression (MLR) $-$ also called softmax regression $-$ to the Poincar\'e ball.

\paragraph{Reformulating Euclidean MLR.} Let's first reformulate Euclidean MLR from the perspective of distances to margin hyperplanes, as in \cite[Section 5]{lebanon2004hyperplane}. This will allow us to easily generalize it. 

Given $K$ classes, one learns a margin hyperplane for each such class using softmax probabilities:
\begin{align}
\forall k\in\{1,...,K\},\ \quad p(y = k | x) \propto \exp\left(\left(\langle a_k,x\rangle - b_k\right) \right), \quad \text{where\ }  b_k \in \Re,\ x, a_k \in \Re^n.
\label{eq:standard_mlr}
\end{align}

Note that any affine hyperplane in $\Re^n$ can be written with a normal vector $a$ and a scalar shift $b$:
\begin{align}
H_{a,b} = \{ x \in \Re^n : \langle a,x\rangle - b = 0\}, \quad \text{where\ } a \in \Re^n \setminus \{\textbf{0}\},\ \text{and}\ b \in \Re.
\label{eq:eucl_hyperplane}
\end{align}
As in \cite[Section 5]{lebanon2004hyperplane}, we note that $\langle a,x\rangle - b = \text{sign}(\langle a,x\rangle - b)\Vert a\Vert d(x,H_{a,b})$. Using Eq.~(\ref{eq:standard_mlr}):
\begin{align}
p(y = k | x) \propto \exp(\text{sign}(\langle a_k,x\rangle - b_k) \Vert a_k\Vert  d(x,H_{a_k,b_k})),\ b_k \in \Re, x, a_k \in \Re^n.
\label{eq:final_mlr}
\end{align}

As it is not immediately obvious how to generalize the Euclidean hyperplane of Eq.~(\ref{eq:eucl_hyperplane}) to other spaces such as the Poincar\'e ball, we reformulate it as follows:
\begin{align}
\tilde{H}_{a,p} = \{x \in \Re^n : \inp{-p + x}{a} = 0 \} = p + \{a\}^{\perp},\ \text{where\ } p\in \Re^n,\ a\in\Re^n \setminus \{\textbf{0}\}.
\label{eq:eucl_hyp_def}
\end{align}
This new definition relates to the previous one as $\tilde{H}_{a,p}=H_{a,\langle a,p\rangle}$. Rewriting Eq.~(\ref{eq:final_mlr}) with $b=\langle a,p\rangle$:
\begin{align}
 p(y = k | x) \propto \exp(\text{sign}(\langle -p_k+x,a_k\rangle) \Vert a_k\Vert  d(x,\tilde{H}_{a_k,p_k})),\ \text{with}\ p_k , x, a_k \in \Re^n.
\label{eq:final_mlr_bis}
\end{align}
 
It is now natural to adapt the previous definition to the hyperbolic setting by replacing $+$ by $\oplus_c$:
\begin{definition}[Poincar\'e hyperplanes]
For $p\in\D_c^n,\ a\in T_p\D_c^n\setminus\{\textbf{0}\}$, let $\{a\}^{\perp} := \{z \in T_p\D_c^n: g^c_p(z,a) = 0 \}=\{z \in T_p\D_c^n: \langle z,a\rangle = 0 \}$. Then, we define Poincar\'e hyperplanes as
\begin{align}
\tilde{H}_{a,p}^{c} := \{x \in \D^n_c : \inp{\log_p^c(x)}{a}_p = 0 \} = \exp^c_{p}(\{a\}^{\perp})=\{x \in \D^n_c : \inp{-p\oplus_c x}{a} = 0 \}.
\label{eq:poincare_hyp_def}
\end{align}
\end{definition}

The last equality is shown appendix~\ref{sec:poinc_hyp_defs_lemma_proof}. $\tilde{H}_{a,p}^{c}$ can also be described as the union of images of all geodesics in $\D^n_c$ orthogonal to $a$ and containing $p$. Notice that our definition matches that of \textit{hypergyroplanes}, see \cite[definition 5.8]{ungar2014analytic}. A 3D hyperplane example is depicted in Fig.~\ref{fig:hyp}. 

Next, we need the following theorem, proved in appendix~\ref{sec:hyp_dist_lemma_proof}:

\begin{theorem}
\label{thm:hyp_dist_lemma_proof}
\begin{align}
d_c(x,\tilde{H}_{a,p}^{c}) := \inf_{w \in \tilde{H}_{a,p}^{c}} d_c(x,w) = \dfrac{1}{\sqrt{c}}\sinh^{-1}\left( \frac{2\sqrt{c}|\inp{-p\oplus_c x}{a}|}{(1 -c \| -p\oplus_c x \|^2)\Vert a\Vert} \right).
\end{align}
\end{theorem}

\paragraph{Final formula for MLR in the Poincar\'e ball.} Putting together Eq.~(\ref{eq:final_mlr_bis}) and Thm.~\ref{thm:hyp_dist_lemma_proof}, we get the hyperbolic MLR formulation. Given $K$ classes and $k\in\{1,\ldots, K\},\ p_k \in \D^n_c,\ a_k\in T_{p_k}\D^n_c\setminus\{\textbf{0}\}$:
\begin{align}
p(y = k | x) \propto \exp(\text{sign}(\inp{-p_k\oplus_c x}{a_k})\sqrt{g^c_{p_k}(a_k,a_k)}d_c(x,\tilde{H}^{c}_{a_k,p_k})), \quad \forall x \in \D^n_c,
\end{align}
or, equivalently
\begin{align}
p(y = k | x) \propto \exp\left(\dfrac{\lambda^c_{p_k}\Vert a_k\Vert}{\sqrt{c}} \sinh^{-1}\left( \frac{2\sqrt{c}\inp{-p_k \oplus_c x}{a_k}}{(1 - c\| -p_k \oplus_c x \|^2)\Vert a_k\Vert} \right) \right), \quad \forall x \in \D^n_c.
\label{eq:hyperbolic_mlr}
\end{align}
Notice that when $c$ goes to zero, this goes to $p(y = k | x) \propto \exp(4\inp{-p_k + x}{a_k})=\exp((\lambda^0_{p_k})^2\inp{-p_k + x}{a_k})=\exp(\inp{-p_k + x}{a_k}_0)$, recovering the usual Euclidean softmax. 

However, at this point it is unclear how to perform optimization over $a_k$, since it lives in $T_{p_k}\D_c^n$ and hence depends on $p_k$. The solution is that one should write $a_k=P^c_{\textbf{0}\to p_k}(a'_k)=(\lambda_\textbf{0}^c/\lambda_{p_k}^c)a'_k$, where $a'_k\in T_\textbf{0}\D_c^n=\Re^n$, and optimize $a'_k$ as a Euclidean parameter. 

\newpage
\subsection{Hyperbolic feed-forward layers}\label{sec:hyp_ffl}

\begin{wrapfigure}{r}{0.35\textwidth}
  \vspace{-55pt}
  \begin{center}
    \includegraphics[width=0.35\textwidth]{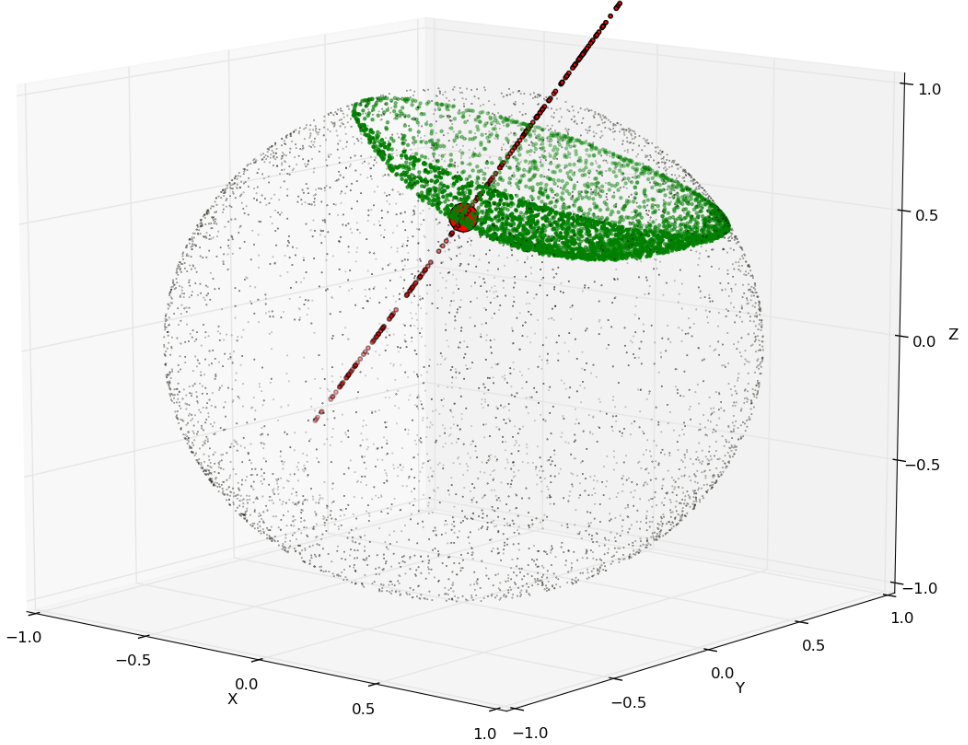}
  \end{center}
  \vspace{-10pt}
  \caption{\textit{An example of a hyperbolic hyperplane in $\D_1^3$ plotted using sampling. The red point is $p$. The shown normal axis to the hyperplane through $p$  is parallel to $a$.}}
  \label{fig:hyp}
  \vspace{-0pt}
\end{wrapfigure}

In order to define hyperbolic neural networks, it is crucial to define a canonically simple parametric family of transformations, playing the role of linear mappings in usual Euclidean neural networks, and to know how to apply pointwise non-linearities. Inspiring ourselves from our reformulation of M\"obius scalar multiplication in Eq.~(\ref{eq:scalar_otimes_exp}), we define:

\begin{definition}[M\"obius version]
For $f:\Re^n\to\Re^m$, we define the \textit{M\"obius version of $f$} as the map from $\D_c^n$ to $\D_c^m$ by:
\begin{equation}\label{eq:mobius_version}
f^{\otimes_c}(x) :=\exp_\textbf{0}^c(f(\log_\textbf{0}^c(x))),
\end{equation}
where $\exp_\textbf{0}^c:T_{\textbf{0}_m}\D_c^m\to \D_c^m$ and $\log_\textbf{0}^c:\D_c^n\to T_{\textbf{0}_n}\D_c^n$.
\end{definition}
Note that similarly as for other M\"obius operations, we recover the Euclidean mapping in the limit $c\to 0$ if $f$ is continuous, as $\lim_{c\to 0}f^{\otimes_c}(x)=f(x)$. This definition satisfies a few desirable properties too, such as: $(f\circ g)^{\otimes_c} = f^{\otimes_c}\circ g^{\otimes_c}$ for $f:\Re^m\to\Re^l$ and $g:\Re^n\to\Re^m$ (morphism property),  and $f^{\otimes_c}(x)/\Vert f^{\otimes_c}(x)\Vert=f(x)/\Vert f(x)\Vert$ for $f(x)\neq\textbf{0}$ (direction preserving). It is then straight-forward to prove the following result:

\begin{lemma}[M\"obius matrix-vector multiplication]
If $M:\Re^n\to\Re^m$ is a linear map, which we identify with its matrix representation, then $ \forall x\in\D_c^n$, if  $Mx \neq \textbf{0}$ we have
\begin{equation}\label{eq:mobius_mat_mult}
M^{\otimes_c}(x)=(1/\sqrt{c}) \tanh\left(\dfrac{\Vert Mx\Vert}{\Vert x\Vert}\tanh^{-1}(\sqrt{c}\Vert x\Vert)\right)\dfrac{Mx}{\Vert Mx\Vert},
\end{equation}
and $M^{\otimes_c}(x)=\textbf{0}$ if $Mx=\textbf{0}$. Moreover, if we define the M\"obius matrix-vector multiplication of  $M\in\mathcal{M}_{m,n}(\Re)$ and $x\in\D_c^n$ by $M\otimes_c x:=M^{\otimes_c}(x)$, then we have $(MM')\otimes_c x=M\otimes_c(M'\otimes_c x)$ for $M\in\mathcal{M}_{l,m}(\mathbb{R})$ and $M'\in\mathcal{M}_{m,n}(\mathbb{R})$ (matrix associativity), $(rM)\otimes_c x = r \otimes_c (M \otimes_c x)$ for $r\in\Re$ and $M\in\mathcal{M}_{m,n}(\Re)$ (scalar-matrix associativity) and $M\otimes_c x=Mx$ for all $M\in\mathcal{O}_n(\mathbb{R})$ (rotations are preserved).
\end{lemma}

\paragraph{Pointwise non-linearity.} If $\varphi:\Re^n\to\Re^n$ is a pointwise non-linearity, then its M\"obius version $\varphi^{\otimes_c}$ can be applied to elements of the Poincar\'e ball.

\paragraph{Bias translation.} The generalization of a translation in the Poincar\'e ball is naturally given by moving along geodesics. But should we use the M\"obius sum $x\oplus_c b$ with a hyperbolic bias $b$ or the exponential map $\exp_x^c(b')$ with a Euclidean bias $b'$? These views are unified with parallel transport (see Thm~\ref{thm:parallel_transport}). M\"obius translation of a point $x\in\D_c^n$ by a bias $b\in\D_c^n$ is given by
\begin{equation}\label{eq:bias}
x\leftarrow x\oplus_c b=\exp_x^c(P^c_{\textbf{0}\to x}(\log_\textbf{0}^c(b)))=\exp_x^c\left(\dfrac{\lambda_\textbf{0}^c}{\lambda_x^c}\log_\textbf{0}^c(b)\right).
\end{equation}
We recover Euclidean translations in the limit $c\to 0$.
Note that bias translations play a particular role in this model. Indeed, consider multiple layers of the form $f_k(x)=\varphi_k(M_k x)$, each of which having M\"obius version $f_k^{\otimes_c}(x)=\varphi_k^{\otimes_c}(M_k\otimes_c x)$. Then their composition can be re-written $f_k^{\otimes_c}\circ\dots\circ f_1^{\otimes_c}=\exp_\textbf{0}^c\circ f_k \circ\dots\circ f_1\circ\log_\textbf{0}^c$. This means that these operations can essentially be performed in Euclidean space. Therefore, it is the interposition between those with the bias translation of Eq.~(\ref{eq:bias}) which differentiates this model from its Euclidean counterpart.

\paragraph{Concatenation of multiple input vectors.} If a vector $x\in\mathbb{R}^{n+p}$ is the (vertical) concatenation of two vectors $x_1\in\Re^n$, $x_2\in\Re^p$, and $M\in\mathcal{M}_{m,n+p}(\Re)$ can be written as the (horizontal) concatenation of two matrices $M_1\in\mathcal{M}_{m,n}(\Re)$ and $M_2\in\mathcal{M}_{m,p}(\Re)$, then $Mx=M_1x_1+M_2 x_2$. We generalize this to hyperbolic spaces: if we are given $x_1\in\D_c^n$, $x_2\in\D_c^p$, $x=(x_1\ x_2)^T\in\D_c^n\times\D_c^p$, and $M,M_1,M_2$ as before, then we define $M\otimes_c x:=M_1\otimes_c x_1\oplus_c M_2\otimes_c x_2$. Note that when $c$ goes to zero, we recover the Euclidean formulation, as $\lim_{c\to 0} M\otimes_c x=\lim_{c\to 0}M_1\otimes_c x_1\oplus_c M_2\otimes_c x_2=M_1x_1+M_2 x_2=Mx$. Moreover, hyperbolic vectors $x\in\D_c^n$ can also be "concatenated" with real features $y \in \Re$ by doing: $M \otimes_c x \oplus_c y \otimes_c b$ with learnable $b \in\D_c^m$ and $M\in\mathcal{M}_{m,n}(\Re)$.

\subsection{Hyperbolic RNN}
\paragraph{Naive RNN.} A simple RNN can be defined by $h_{t+1}=\varphi(W h_t+U x_t+b)$ where $\varphi$ is a pointwise non-linearity, typically $\tanh$, sigmoid, ReLU, etc. This formula can be naturally generalized to the hyperbolic space as follows. For parameters $W\in\mathcal{M}_{m,n}(\Re)$, $U\in\mathcal{M}_{m,d}(\Re)$, $b\in\D_c^m$, we define:
\begin{equation}\label{eq:naive_hyp_rnn}
h_{t+1}=\varphi^{\otimes_c}(W\otimes_c h_t\oplus_c U\otimes_c x_t\oplus_c b),\quad h_t\in\D_c^n,\ x_t\in\D_c^d.
\end{equation}
Note that if inputs $x_t$'s are Euclidean, one can write $\tilde{x}_t:=\exp_0^c(x_t)$ and use the above formula, since $\exp^c_{W\otimes_c h_t}(P^c_{\textbf{0}\to W\otimes_c h_t}(U x_t))=W\otimes_c h_t\oplus_c\exp^c_\textbf{0}(U x_t)=W\otimes_c h_t\oplus_c U\otimes_c \tilde{x}_t$.

\paragraph{GRU architecture.} One can also adapt the GRU architecture:
\begin{equation}\label{eq:GRU}
	\begin{aligned}
		r_t &=\sigma(W^{r} h_{t-1}+U^{r}x_t+b^r), & z_t &=\sigma(W^{z} h_{t-1}+U^{z}x_t+b^z),\\
		\tilde{h}_t &=\varphi(W (r_t\odot h_{t-1})+U x_t+b), & h_t &=(1-z_t)\odot h_{t-1}+z_t\odot\tilde{h}_t,
	\end{aligned}
\end{equation}
where $\odot$ denotes pointwise product. First, how should we adapt the pointwise multiplication by a scaling gate? Note that the definition of the M\"obius version (see Eq.~(\ref{eq:mobius_version})) can be naturally extended to maps $f:\Re^n\times\Re^p\to\Re^m$ as $f^{\otimes_c}:(h,h')\in\D^n_c\times\D^p_c\mapsto\exp_\textbf{0}^c(f(\log_\textbf{0}^c(h),\log_\textbf{0}^c(h')))$. In particular, choosing $f(h,h'):=\sigma(h)\odot h'$ yields\footnote{If $x$ has $n$ coordinates, then $\text{diag}(x)$ denotes the diagonal matrix of size $n$ with $x_i$'s on its diagonal.} $f^{\otimes_c}(h,h')=\exp_\textbf{0}^c(\sigma(\log_\textbf{0}^c(h))\odot\log_\textbf{0}^c(h'))=\text{diag}(\sigma(\log_\textbf{0}^c(h)))\otimes_c h'$. Hence we adapt $r_t\odot h_{t-1}$ to $\text{diag}(r_t)\otimes_c h_{t-1}$ and the reset gate $r_t$ to:
\begin{equation}\label{eq:hyp_scaling}
r_t =\sigma\log^c_\textbf{0}(W^{r}\otimes_c h_{t-1}\oplus_c U^{r}\otimes_c x_t\oplus_c b^r),
\end{equation}
and similarly for the update gate $z_t$. Note that as the argument of $\sigma$ in the above is unbounded, $r_t$ and $z_t$ can a priori take values onto the full range $(0,1)$. Now the intermediate hidden state becomes:
\begin{equation}
\tilde{h}_t =\varphi^{\otimes_c}((W\text{diag}(r_t))\otimes_c h_{t-1}\oplus_c U\otimes_c x_t\oplus b),
\end{equation}
where M\"obius matrix associativity simplifies $W\otimes_c(\text{diag}(r_t)\otimes_c h_{t-1})$ into $(W\text{diag}(r_t))\otimes_c h_{t-1}$. Finally, we propose to adapt the update-gate equation as
\begin{equation}\label{eq:gru_update}
h_t =h_{t-1}\oplus_c\text{diag}(z_t)\otimes_c(-h_{t-1}\oplus_c\tilde{h}_t).
\end{equation}
Note that when $c$ goes to zero, one recovers the usual GRU. Moreover, if $z_t=\textbf{0}$ or $z_t=\textbf{1}$, then $h_t$ becomes $h_{t-1}$ or $\tilde{h}_t$ respectively, similarly as in the usual GRU. This adaptation was obtained by adapting \cite{tallec2018can}: in this work, the authors re-derive the update-gate mechanism from a first principle called \textit{time-warping invariance}. We adapted their derivation to the hyperbolic setting by using the notion of \textit{gyroderivative} \cite{birman2001hyperbolic} and proving a \textit{gyro-chain-rule} (see appendix~\ref{sec:hyp_GRU}).

\section{Experiments}

\paragraph{SNLI task and dataset.} We evaluate our method on two tasks. The first is natural language inference, or textual entailment. Given two sentences, a premise (e.g. "Little kids A. and B. are playing soccer.") and a hypothesis (e.g. "Two children are playing outdoors."), the binary classification task is to predict whether the second sentence can be inferred from the first one. This defines a partial order in the sentence space. We test hyperbolic networks on the biggest real dataset for this task, SNLI~\citep{bowman2015large}. It consists of 570K training, 10K validation and 10K test sentence pairs. Following~\citep{vendrov2015order}, we merge the "contradiction" and "neutral" classes into a single class of negative sentence pairs, while the "entailment" class gives the positive pairs.

\paragraph{PREFIX task and datasets.} We conjecture that the improvements of hyperbolic neural networks are more significant when the underlying data structure is closer to a tree. To test this, we design a proof-of-concept task of \textit{detection of noisy prefixes}, i.e. given two sentences, one has to decide if the second sentence is a noisy prefix of the first, or a random sentence. We thus build synthetic datasets PREFIX-Z\% (for Z being 10, 30 or 50) as follows: for each random first sentence of random length at most 20 and one random prefix of it, a second positive sentence is generated by randomly replacing Z\% of the words of the prefix, and a second negative sentence of same length is randomly generated. Word vocabulary size is 100, and we generate 500K training, 10K validation and 10K test pairs. 

\paragraph{Models architecture.} Our neural network layers can be used in a plug-n-play manner exactly like standard Euclidean layers. They can also be combined with Euclidean layers. However, optimization w.r.t. hyperbolic parameters is different (see below) and based on Riemannian gradients which are just rescaled Euclidean gradients when working in the conformal Poincar\'e model~\citep{nickel2017poincar}. Thus, back-propagation can be applied in the standard way. 

In our setting, we embed the two sentences using two distinct hyperbolic RNNs or GRUs. The sentence embeddings are then fed together with their squared distance (hyperbolic or Euclidean, depending on their geometry) to a FFNN (Euclidean or hyperbolic, see Sec.~\ref{sec:hyp_ffl}) which is further fed to an MLR (Euclidean or hyperbolic, see Sec.~\ref{sec:hyp_mlr})  that gives probabilities of the two classes (entailment vs neutral). We use cross-entropy loss on top. Note that hyperbolic and Euclidean layers can be mixed, e.g. the full network can be hyperbolic and only the last layer be Euclidean, in which case one has to use $\log_\textbf{0}$ and $\exp_\textbf{0}$ functions to move between the two manifolds in a correct manner as explained for Eq.~\ref{eq:mobius_version}.

\paragraph{Optimization.} Our models have both Euclidean (e.g. weight matrices in both Euclidean and hyperbolic FFNNs, RNNs or GRUs) and hyperbolic parameters (e.g. word embeddings or biases for the hyperbolic layers). We optimize the Euclidean parameters with Adam~\citep{kingma2014adam} (learning rate 0.001). Hyperbolic parameters cannot be updated with an equivalent method that keeps track of gradient history due to the absence of a Riemannian Adam. Thus, they are optimized using full Riemannian stochastic gradient descent (RSGD)~\citep{bonnabel2013stochastic,ganea2018hyperbolic}. We also experiment with projected RSGD~\citep{nickel2017poincar}, but optimization was sometimes less stable. We use a different constant learning rate for word embeddings (0.1) and other hyperbolic weights (0.01) because words are updated less frequently. 

\paragraph{Numerical errors.} Gradients of the basic operations defined above (e.g. $\oplus_c$, exponential map) are not defined when the hyperbolic argument vectors are on the ball border, i.e. $\sqrt{c} \|x\| =1$. Thus, we always project results of these operations in the ball of radius $1 - \epsilon$, where $\epsilon = 10^{-5}$. Numerical errors also appear when hyperbolic vectors get closer to $\textbf{0}$, thus we perturb them with an $\epsilon' = 10^{-15}$ before they are used in any of the above operations. Finally, arguments of the $\tanh$ function are clipped between $\pm 15$ to avoid numerical errors, while arguments of $\tanh^{-1}$ are clipped to at most $1 - 10^{-5}$.

\paragraph{Hyperparameters.} For all methods, baselines and datasets, we use $c=1$, word and hidden state embedding dimension of 5 (we focus on the low dimensional setting that was shown to already be effective~\citep{nickel2017poincar}), batch size of 64. We ran all methods for a fixed number of $30$ epochs. For all models, we experiment with both \textit{identity} (no non-linearity) or $\tanh$ non-linearity in the RNN/GRU cell, as well as \textit{identity} or ReLU after the FFNN layer and before MLR. As expected, for the fully Euclidean models, $\tanh$ and ReLU respectively surpassed the \textit{identity} variant by a large margin. We only report the best Euclidean results. Interestingly, for the hyperbolic models, using only identity for both non-linearities works slightly better and this is likely due to two facts: i) our hyperbolic layers already contain non-linearities by their nature, ii) $\tanh$ is limiting the output domain of the sentence embeddings, but the hyperbolic specific geometry is more pronounced at the ball border, i.e. at the hyperbolic "infinity", compared to the center of the ball. 

For the results shown in Tab.~\ref{tab:results}, we run each model (baseline or ours) exactly 3 times and report the test result corresponding to the best validation result from these 3 runs. We do this because the highly non-convex spectrum of hyperbolic neural networks sometimes results in convergence to poor local minima, suggesting that initialization is very important.

\begin{table*}[t]
\begin{center}
\begin{small}
\begin{sc}
\begin{tabular}{| c || c | c | c | c |}
\toprule

                               &   SNLI  &  PREFIX-10\%  &   PREFIX-30\% &   PREFIX-50\% \\
\midrule
Fully Euclidean RNN            &    \textbf{79.34} \% &  89.62 \% & 81.71 \% & 72.10 \% \\ \hline
Hyperbolic RNN+FFNN, Eucl MLR  &   \textbf{79.18}  \% &  96.36  \% & \textbf{87.83} \% &  \textbf{76.50 } \% \\ \hline 
Fully Hyperbolic RNN           &    78.21 \% &   \textbf{96.91}    \% &   87.25  \% &   62.94 \%    \\ \hline \hline

Fully Euclidean GRU            &   \textbf{81.52} \% & 95.96 \% &  86.47 \% &   75.04    \% \\ \hline 
Hyperbolic GRU+FFNN, Eucl MLR  &   79.76 \% & \textbf{97.36}  \% & \textbf{88.47}  \% & \textbf{76.87}  \% \\ \hline 
Fully Hyperbolic GRU           &    \textbf{81.19}  \% &  \textbf{97.14} \% &  \textbf{88.26} \% &  \textbf{76.44}   \% \\ \hline

\end{tabular}
\end{sc}
\end{small}
\end{center}
\vskip -0.1in
\caption{Test accuracies for various models and four datasets. "Eucl" denotes Euclidean. All word and sentence embeddings have dimension 5. We highlight in \textbf{bold} the best baseline (or baselines, if the difference is less than 0.5\%).}
\label{tab:results}
\end{table*}

\paragraph{Results.} Results are shown in Tab.~\ref{tab:results}. Note that the fully Euclidean baseline models might have an advantage over hyperbolic baselines because more sophisticated optimization algorithms such as Adam do not have a hyperbolic analogue at the moment. We first observe that all GRU models overpass their RNN variants. Hyperbolic RNNs and GRUs have the most significant improvement over their Euclidean variants when the underlying data structure is more tree-like, e.g. for PREFIX-10\% $-$ for which the tree relation between sentences and their prefixes is more prominent $-$ we reduce the error by a factor of $3.35$ for hyperbolic vs Euclidean RNN, and by a factor of $1.5$ for hyperbolic vs Euclidean GRU. As soon as the underlying structure diverges more and more from a tree, the accuracy gap decreases $-$ for example, for PREFIX-50\% the noise heavily affects the representational power of hyperbolic networks. Also, note that on SNLI our methods perform similarly as with their Euclidean variants. Moreover, hyperbolic and Euclidean MLR are on par when used in conjunction with hyperbolic sentence embeddings, suggesting further empirical investigation is needed for this direction (see below).

We also observe that, in the hyperbolic setting, accuracy tends to increase when sentence embeddings start increasing, and gets better as their norms converge towards 1 (the ball border for $c=1$). Unlike in the Euclidean case, this behavior does happen only after a few epochs and suggests that the model should first adjust the angular layout in order to disentangle the representations, before increasing their norms to fully exploit the strong clustering property of the hyperbolic geometry. Similar behavior was observed in the context of embedding trees by~\citep{nickel2017poincar}. Details in appendix~\ref{sec:exp_res}.

\begin{wrapfigure}{r}{0.61\textwidth}
  \vspace{-10pt}
  \begin{center}
    \includegraphics[width=0.3\textwidth]{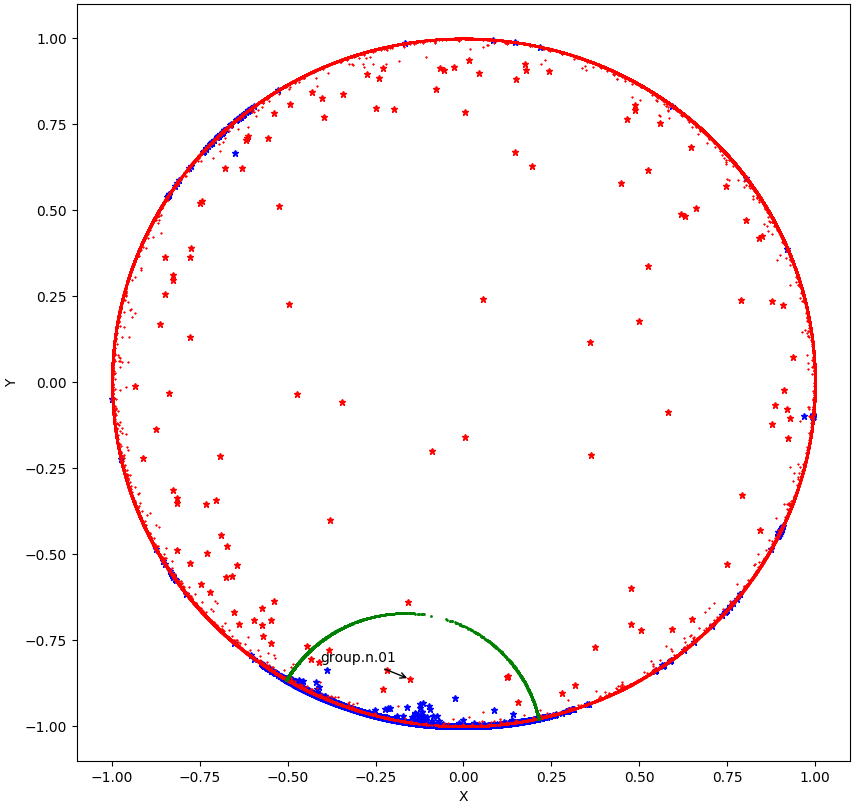}
    \includegraphics[width=0.3\textwidth]{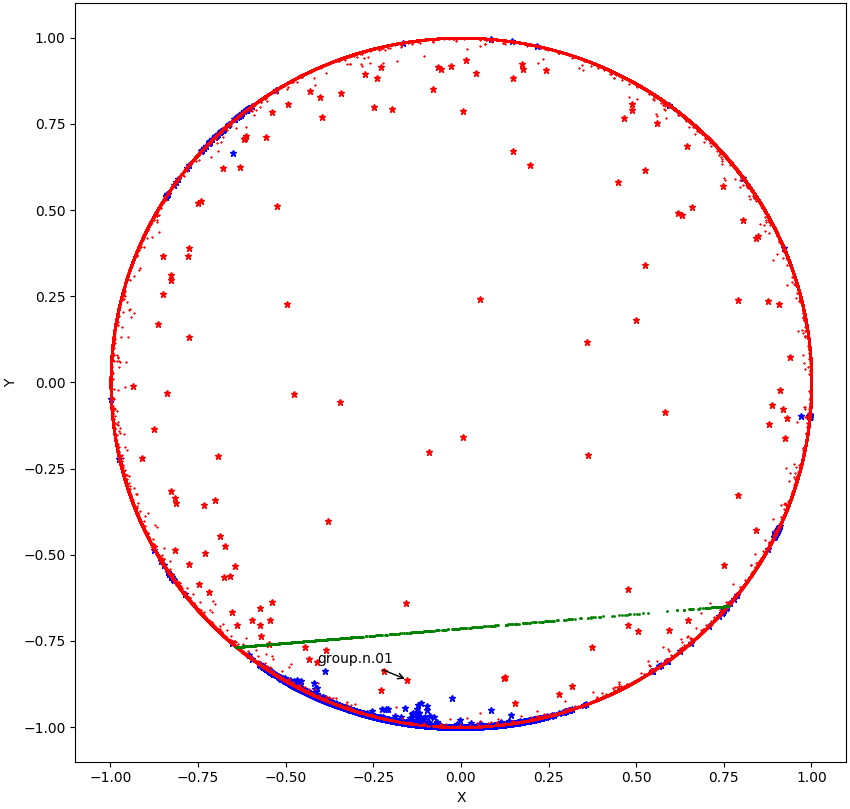}
  \end{center}
  \vspace{-0pt}
  \caption{\textit{Hyperbolic (left) vs Direct Euclidean (right) binary MLR used to classify nodes as being part in the \textsc{group.n.01} subtree of the WordNet noun hierarchy solely based on their Poincar\'e embeddings. The positive points (from the subtree) are in blue, the negative points (the rest) are in red and the trained positive separation hyperplane is depicted in green.}}
  \label{fig:mlr_exp}
  \vspace{-0pt}
\end{wrapfigure}

\paragraph{MLR classification experiments.} For the sentence entailment classification task we do not see a clear advantage of hyperbolic MLR compared to its Euclidean variant. A possible reason is that, when trained end-to-end, the model might decide to place positive and negative embeddings in a manner that is already well separated with a classic MLR. As a consequence, we further investigate MLR for the task of subtree classification. Using an open source implementation\footnote{\url{https://github.com/dalab/hyperbolic_cones}} of \citep{nickel2017poincar}, we pre-trained Poincar\'e embeddings of the WordNet noun hierarchy (82,115 nodes). We then choose one node in this tree (see Table~\ref{tab:results_mlr}) and classify all other nodes (solely based on their embeddings) as being part of the subtree rooted at this node. All nodes in such a subtree are divided into positive training nodes (80\%) and positive test nodes (20\%). The same splitting procedure is applied for the remaining WordNet nodes that are divided into a negative training and negative test set respectively. Three variants of MLR are then trained on top of pre-trained Poincar\'e embeddings~\citep{nickel2017poincar} to solve this binary classification task: hyperbolic MLR, Euclidean MLR applied directly on the hyperbolic embeddings and Euclidean MLR applied after mapping all embeddings in the tangent space at \textbf{0} using the $\log_\textbf{0}$ map. We use different embedding dimensions : 2, 3, 5 and 10. For the hyperbolic MLR, we use full Riemannian SGD with a learning rate of 0.001. For the two Euclidean models we use ADAM optimizer and the same learning rate. During training, we always sample the same number of negative and positive nodes in each minibatch of size 16; thus positive nodes are frequently resampled. All methods are trained for 30 epochs and the final F1 score is reported (no hyperparameters to validate are used, thus we do not require a validation set). This procedure is repeated for four subtrees of different sizes.

Quantitative results are presented in Table~\ref{tab:results_mlr}. We can see that the hyperbolic MLR overpasses its Euclidean variants in almost all settings, sometimes by a large margin. Moreover, to provide further understanding, we plot the 2-dimensional embeddings and the trained separation hyperplanes (geodesics in this case) in Figure~\ref{fig:mlr_exp}. We can see that respecting the hyperbolic geometry is very important for a quality classification model.

\begin{table*}[t]
\hspace{-1cm}
\begin{small}
\begin{sc}
\begin{tabular}{| c || c || c | c | c | c |}
\toprule

\begin{tabular}{@{}c@{}} {WordNet}\\ {subtree} \end{tabular} & Model  & d = 2  &  d = 3  &  d = 5  & d = 10 \\
\midrule
\begin{tabular}{@{}c@{}} {animal.n.01}\\ {3218 / 798} \end{tabular} & \begin{tabular}{@{}c@{}} {Hyperbolic}\\ {Direct Eucl} \\ {$\log_\textbf{0}$ + Eucl} \end{tabular}  &  \begin{tabular}{@{}c@{}} {$\mathbf{47.43 \pm 1.07\%}$ }\\ {$41.69 \pm 0.19\%$ } \\ {$38.89 \pm 0.01\%$ } \end{tabular}  & \begin{tabular}{@{}c@{}} {$\mathbf{91.92 \pm 0.61\%}$ }\\ {$68.43 \pm 3.90\%$ } \\ {$62.57 \pm 0.61\%$ } \end{tabular}  & \begin{tabular}{@{}c@{}} {$\mathbf{98.07 \pm 0.55\%}$ }\\ {$95.59 \pm 1.18\%$ } \\ {$89.21 \pm 1.34\%$ } \end{tabular}  & \begin{tabular}{@{}c@{}} {$\mathbf{99.26 \pm 0.59\%}$ }\\ {$\mathbf{99.36 \pm 0.18\%}$ } \\ { $98.27 \pm 0.70\%$} \end{tabular} \\ \hline

\begin{tabular}{@{}c@{}} {group.n.01}\\ {6649 / 1727} \end{tabular}  & \begin{tabular}{@{}c@{}} {Hyperbolic}\\ {Direct Eucl} \\ {$\log_\textbf{0}$ + Eucl} \end{tabular}  &  \begin{tabular}{@{}c@{}} {$\mathbf{81.72 \pm 0.17\%}$ }\\ {$61.13 \pm 0.42\%$} \\ { $60.75 \pm 0.24\%$} \end{tabular}  & \begin{tabular}{@{}c@{}} {$\mathbf{89.87 \pm 2.73\%}$ }\\ {$63.56 \pm 1.22\%$ } \\ {$61.98 \pm 0.57\%$ } \end{tabular} & \begin{tabular}{@{}c@{}} {$\mathbf{87.89 \pm 0.80\%}$ }\\ {$67.82 \pm 0.81\%$ } \\ {$67.92 \pm 0.74\%$ } \end{tabular} & \begin{tabular}{@{}c@{}} {$\mathbf{91.91 \pm 3.07\%}$ }\\ {$\mathbf{91.38 \pm 1.19\%}$ } \\ {$\mathbf{91.41 \pm 0.18\%}$ } \end{tabular} \\ \hline 

\begin{tabular}{@{}c@{}} {worker.n.01}\\ {861 / 254} \end{tabular}  & \begin{tabular}{@{}c@{}} {Hyperbolic}\\ {Direct Eucl} \\ {$\log_\textbf{0}$ + Eucl} \end{tabular}  & \begin{tabular}{@{}c@{}} {$\mathbf{12.68 \pm 0.82\%}$}\\ { $10.86 \pm 0.01\%$} \\ {$9.04 \pm 0.06\%$ } \end{tabular}   & \begin{tabular}{@{}c@{}} {$\mathbf{24.09 \pm 1.49\%}$ }\\ { $22.39 \pm 0.04\%$} \\ {$22.57 \pm 0.20\%$ } \end{tabular} & \begin{tabular}{@{}c@{}} {$\mathbf{55.46 \pm 5.49\%}$ }\\ {$35.23 \pm 3.16\%$ } \\ { $26.47 \pm 0.78\%$} \end{tabular} & \begin{tabular}{@{}c@{}} {$\mathbf{66.83 \pm 11.38\%}$ }\\ {$47.29 \pm 3.93\%$ } \\ {$36.66 \pm 2.74\%$ } \end{tabular} \\ \hline 

\begin{tabular}{@{}c@{}} {mammal.n.01}\\ {953 / 228} \end{tabular}  & \begin{tabular}{@{}c@{}} {Hyperbolic}\\ {Direct Eucl} \\ {$\log_\textbf{0}$ + Eucl} \end{tabular}  &  \begin{tabular}{@{}c@{}} {$\mathbf{32.01 \pm 17.14\%}$ }\\ { $\mathbf{15.58 \pm 0.04\%}$} \\ {$13.10 \pm 0.13\%$ } \end{tabular}  & \begin{tabular}{@{}c@{}} {$\mathbf{87.54 \pm 4.55\%}$ }\\ {$44.68 \pm 1.87\%$ } \\ {$44.89 \pm 1.18\%$ } \end{tabular} & \begin{tabular}{@{}c@{}} {$\mathbf{88.73 \pm 3.22\%}$ }\\ {$59.35 \pm 1.31\%$ } \\ {$52.51 \pm 0.85\%$ } \end{tabular} & \begin{tabular}{@{}c@{}} {$\mathbf{91.37 \pm 6.09\%}$ }\\ { $77.76 \pm 5.08\%$} \\ {$56.11 \pm 2.21\%$ } \end{tabular} \\ \hline 

\end{tabular}
\end{sc}
\end{small}
\caption{Test F1 classification scores for four different subtrees of WordNet noun tree. All nodes in such a subtree are divided into positive training nodes (80\%) and positive test nodes (20\%); these counts are shown below each subtree root. The same splitting procedure is applied for the remaining nodes to obtain negative training and test sets. Three variants of MLR are then trained on top of pre-trained Poincar\'e embeddings~\citep{nickel2017poincar} to solve this binary classification task: hyperbolic MLR, Euclidean MLR applied directly on the hyperbolic embeddings and Euclidean MLR applied after mapping all embeddings in the tangent space at \textbf{0} using the $\log_\textbf{0}$ map. 95\% confidence intervals for 3 different runs are shown for each method and each different embedding dimension (2, 3, 5 or 10).}
\vskip -0.1in
\label{tab:results_mlr}
\end{table*}


\section{Conclusion}

We showed how classic Euclidean deep learning tools such as MLR, FFNNs, RNNs or GRUs can be generalized in a principled manner to all spaces of constant negative curvature combining Riemannian geometry with the elegant theory of gyrovector spaces. Empirically we found that our models outperform or are on par with corresponding Euclidean architectures on sequential data with implicit hierarchical structure. We hope to trigger exciting future research related to better understanding of the hyperbolic non-convexity spectrum and development of other non-Euclidean deep learning methods. 

Our data and Tensorflow~\citep{abadi2016tensorflow} code are publicly available\footnote{\url{https://github.com/dalab/hyperbolic_nn}}.

\section*{Acknowledgements}
We thank Igor Petrovski for useful pointers regarding the implementation. 

This research is funded by the Swiss National Science Foundation (SNSF) under grant agreement number 167176. Gary B\'ecigneul is also funded by the Max Planck ETH Center for Learning Systems.

\bibliographystyle{plain}
\bibliography{th}

\newpage
\appendix

\section{Hyperbolic Trigonometry}\label{sec:hyp_trig}

\paragraph{Hyperbolic angles.} For $A, B, C \in \D^n_c$, we denote by $\angle A := \angle BAC$ the angle between the two geodesics starting from $A$ and ending at $B$ and $C$ respectively. This angle can be defined in two equivalent ways: i) either using the angle between the initial velocities of the two geodesics as given by Eq.~\ref{eq:hyp_angle_tangent_space}, or ii) using the formula
\begin{equation}\label{eq:gyro_angles}
\cos(\angle A)= \left\langle \dfrac{(-A)\oplus_c B}{\Vert (-A)\oplus_c B\Vert},\dfrac{(-A)\oplus_c C}{\Vert (-A)\oplus_c C\Vert} \right\rangle,
\end{equation}
In this case, $\angle A$ is also called a \textit{gyroangle} in the work of \citep[section 4]{ungar2008gyrovector}. \\

\paragraph{Hyperbolic law of sines.} We state here the hyperbolic law of sines. If for $A, B, C \in \D^n_c$, we denote by $\angle B := \angle ABC$ the angle between the two geodesics starting from $B$ and ending at $A$ and $C$ respectively, and by $\tilde{c} = d_c(B,A)$ the length of the hyperbolic segment BA (and similarly for others), then we have:
\begin{align}\label{eq:sine_law}
\frac{\sin (\angle A)}{\sinh(\sqrt{c}\tilde{a})} = \frac{\sin (\angle B)}{\sinh(\sqrt{c}\tilde{b})}  = \frac{\sin (\angle C)}{\sinh(\sqrt{c}\tilde{c})}.
\end{align}

Note that one can also adapt the hyperbolic law of cosines to the hyperbolic space.

\section{Proof of Theorem~\ref{thm:parallel_transport}}\label{sec:parallel_transport}
\textbf{Theorem~\ref{thm:parallel_transport}.}\\
\textit{In the manifold $(\D_c^n,g^c)$, the parallel transport w.r.t. the Levi-Civita connection of a vector $v\in T_\textbf{0}\D_c^n$ to another tangent space $T_x\D_c^n$ is given by the following isometry:}
\begin{equation}
P^c_{\textbf{0}\to x}(v)=\log^c_x(x\oplus_c\exp_\textbf{0}^c(v))=\dfrac{\lambda_\textbf{0}^c}{\lambda_x^c}v.
\end{equation}
\begin{proof}
The geodesic in $\D_c^n$ from $\textbf{0}$ to $x$ is given in Eq.~(\ref{eq:geodesic_2pts}) by $\gamma(t)= x\otimes_c t$, for $t\in[0,1]$. Let $v\in T_\textbf{0}\D_c^n$. Then it is of common knowledge that there exists a unique parallel\footnote{\textit{i.e.} that $\frac{DX}{\partial t}=0$ for $t\in[0,1]$, where $\frac{D}{\partial t}$ denotes the covariant derivative.} vector field $X$ along $\gamma$ (\textit{i.e.} $X(t)\in T_{\gamma(t)}\D_c^n$, $\forall t\in[0,1]$) such that $X(0)=v$. Let's define:
\begin{equation}
X:t\in[0,1]\mapsto\log^c_{\gamma(t)}(\gamma(t)\oplus_c\exp_\textbf{0}^c(v))\in T_{\gamma(t)}\D_c^n.
\end{equation} 
Clearly, $X$ is a vector field along $\gamma$ such that $X(0)=v$.
Now define 
\begin{equation}
P^c_{\textbf{0}\to x}:v\in T_\textbf{0}\D_c^n\mapsto\log^c_x(x\oplus_c\exp_\textbf{0}^c(v))\in T_x\D_c^n.
\end{equation}
From Eq.~(\ref{eq:gyro_exp_map}), it is easily seen that $P^c_{\textbf{0}\to x}(v)=\frac{\lambda_\textbf{0}^c}{\lambda_x^c}v$, hence $P^c_{\textbf{0}\to x}$ is a linear isometry from $T_\textbf{0}\D_c^n$ to $T_x\D_c^n$. Since $P^c_{\textbf{0}\to x}(v)=X(1)$, it is enough to prove that $X$ is parallel in order to guarantee that $P^c_{\textbf{0}\to x}$ is the parallel transport from $T_\textbf{0}\D_c^n$ to $T_x\D_c^n$. 

Since $X$ is a vector field along $\gamma$, its covariant derivative can be expressed with the Levi-Civita connection $\nabla^c$ associated to $g^c$:
\begin{equation}
\frac{DX}{\partial t}=\nabla^c_{\dot{\gamma}(t)}X.
\end{equation}
Let's compute the Levi-Civita connection from its Christoffel symbols. In a local coordinate system, they can be written as 
\begin{equation}
\Gamma_{jk}^i=\frac{1}{2}(g^c)^{il}(\partial_j g^c_{lk}+\partial_k g^c_{lj}-\partial_l g^c_{jk}),
\end{equation}
where superscripts denote the inverse metric tensor and using Einstein's notations. As $g_{ij}^c=(\lambda^c)^2\delta_{ij}$, at $\gamma(t)\in\D_c^n$ this yields:
\begin{equation}
\Gamma_{jk}^i=c\lambda_{\gamma(t)}^c(\delta_{ik}\gamma(t)_j+\delta_{ij}\gamma(t)_k-\delta_{jk}\gamma(t)_i).
\end{equation}
On the other hand, since $X(t)=(\lambda^c_\textbf{0}/\lambda^c_{\gamma(t)})v$, we have 
\begin{equation}
\nabla^c_{\dot{\gamma}(t)}X=\dot{\gamma}(t)^i\nabla^c_i X=\dot{\gamma}(t)^i\nabla^c_i \left(\dfrac{\lambda^c_\textbf{0}}{\lambda^c_{\gamma(t)}}v\right)= v^j\dot{\gamma}(t)^i\nabla^c_i \left(\dfrac{\lambda^c_\textbf{0}}{\lambda^c_{\gamma(t)}}e_j\right).
\end{equation}
Since $\gamma(t)=(1/\sqrt{c})\tanh(t\tanh^{-1}(\sqrt{c}\Vert x\Vert))\frac{x}{\Vert x\Vert}$, it is easily seen that $\dot{\gamma}(t)$ is colinear to $\gamma(t)$. Hence there exists $K^x_t\in\Re$ such that $\dot{\gamma}(t)=K^x_t\gamma(t)$. Moreover, we have the following Leibniz rule:
\begin{equation}
\nabla^c_i \left(\dfrac{\lambda^c_\textbf{0}}{\lambda^c_{\gamma(t)}}e_j\right)=\dfrac{\lambda^c_\textbf{0}}{\lambda^c_{\gamma(t)}}\nabla_i^c e_j + \dfrac{\partial}{\partial\gamma(t)_i}\left(\dfrac{\lambda^c_\textbf{0}}{\lambda^c_{\gamma(t)}}\right)e_j.
\end{equation}
Combining these yields 
\begin{equation}
\frac{DX}{\partial t}=K^x_t v^j\gamma(t)^i\left(\dfrac{\lambda^c_\textbf{0}}{\lambda^c_{\gamma(t)}}\nabla_i^c e_j + \dfrac{\partial}{\partial\gamma(t)_i}\left(\dfrac{\lambda^c_\textbf{0}}{\lambda^c_{\gamma(t)}}\right)e_j\right).
\end{equation}
Replacing with the Christoffel symbols of $\nabla^c$ at $\gamma(t)$ gives
\begin{equation}
\dfrac{\lambda^c_\textbf{0}}{\lambda^c_{\gamma(t)}}\nabla^c_i e_j=
\dfrac{\lambda^c_\textbf{0}}{\lambda^c_{\gamma(t)}}\Gamma_{ij}^k e_k=2c[\delta_j^k\gamma(t)_i+\delta_i^k\gamma(t)_j-\delta_{ij}\gamma(t)^k]e_k.
\end{equation}
Moreover, 
\begin{equation}
\dfrac{\partial}{\partial\gamma(t)_i}\left(\dfrac{\lambda^c_\textbf{0}}{\lambda^c_{\gamma(t)}}\right)e_j=\dfrac{\partial}{\partial\gamma(t)_i}\left(-c\Vert\gamma(t)\Vert^2\right)e_j=-2c\gamma(t)_i e_j.
\end{equation}
Putting together everything, we obtain
\begin{align}
\frac{DX}{\partial t} &= K^x_t v^j\gamma(t)^i\left(2c[\delta_j^k\gamma(t)_i+\delta_i^k\gamma(t)_j-\delta_{ij}\gamma(t)^k]e_k-2c\gamma(t)_i e_j\right)\\
&=2cK^x_t v^j\gamma(t)^i\left(\gamma(t)_je_i-\delta_{ij}\gamma(t)^ke_k\right)\\
&=2cK^x_t v^j\left(\gamma(t)_j\gamma(t)^i e_i-\gamma(t)^i\delta_{ij}\gamma(t)^ke_k\right)\\
&=2cK^x_t v^j\left(\gamma(t)_j\gamma(t)^i e_i-\gamma(t)_j\gamma(t)^ke_k\right)\\
&=0,
\end{align}
which concludes the proof.
\end{proof}

\section{Proof of Eq.~(\ref{eq:poincare_hyp_def})}\label{sec:poinc_hyp_defs_lemma_proof}
\begin{proof}
Two steps proof:

\textit{i)} $\exp^c_{p}(\{a\}^{\perp}) \subseteq \{x \in \D^n_c : \inp{-p \oplus_c x}{a} = 0 \} $:

Let $z \in \{a\}^{\perp}$. From Eq.~(\ref{eq:gyro_exp_map}), we have that:
\begin{align}
\exp^c_{p}(z) = -p \oplus_c \beta z, \quad \text{for\ some\ } \beta \in \Re.
\end{align}
This, together with the left-cancellation law in gyrospaces (see section~\ref{sec:gyro_sec}), implies that 
\begin{align}
\inp{-p \oplus_c \exp^c_{p}(z)}{a} = \inp{\beta z}{a} = 0
\end{align}
which is what we wanted. \\

\textit{ii)} $\{x \in \D^n_c : \inp{-p \oplus_c x}{a} = 0 \} \subseteq \exp^c_{p}(\{a\}^{\perp})$: 

Let $x \in \D^n_c$ s.t. $\inp{-p \oplus_c x}{a} = 0$. Then, using Eq.~(\ref{eq:gyro_exp_map}), we derive that: 
\begin{align}
\log_{p}^c(x) = \beta (-p\oplus_c x), \quad \text{for\ some\ } \beta \in \Re,
\end{align}
which is orthogonal to $a$, by assumption. This implies $\log_{p}^c(x) \in \{a\}^{\perp}$, hence $x \in \exp^c_{p}(\{a\}^{\perp})$.
\end{proof}

\section{Proof of Theorem~\ref{thm:hyp_dist_lemma_proof}}\label{sec:hyp_dist_lemma_proof}
\textbf{Theorem~\ref{thm:hyp_dist_lemma_proof}.}\\
\begin{align}
d_c(x,\tilde{H}_{a,p}^{c}) := \inf_{w \in \tilde{H}_{a,p}^{c}} d_c(x,w) = \dfrac{1}{\sqrt{c}}\sinh^{-1}\left( \frac{2\sqrt{c}|\inp{-p\oplus_c x}{a}|}{(1 -c \| -p\oplus_c x \|^2)\Vert a\Vert} \right).
\end{align}
\textit{Proof.}
We  first need to prove the following lemma, trivial in the Euclidean space, but not in the Poincar\'e ball:

\begin{lemma} (Orthogonal projection on a geodesic)
\label{lemma:proj_geodesic}
Any point in the Poincar\'e ball has a unique orthogonal projection on any given geodesic that does not pass through the point. Formally, for all $y \in \D^n_c$ and for all geodesics $\gamma_{x \rightarrow z}(\cdot)$ s.t. $y \notin \text{Im\ } \gamma_{x \rightarrow z}$, there exists an unique $w \in \text{Im\ } \gamma_{x \rightarrow z}$ s.t. $\angle (\gamma_{w \rightarrow y}, \gamma_{x \rightarrow z}) = \pi/2$.
\end{lemma}

\begin{proof}
We first note that any geodesic in $\D^n_c$ has the form $\gamma(t) = u \oplus_c v \otimes_c t$ as given by Eq.~\ref{eq:gyro_unitspeed_geodesic}, and has two "points at infinity" lying on the ball border $(v\neq\textbf{0}$):
\begin{align}
\gamma(\pm \infty) = u \oplus_c \frac{ \pm v}{\sqrt{c}\|v\|} \in \partial \D^n_c.
\label{eq:pts_at_infinity_geodesic}
\end{align}

Using the notations in the lemma statement, the closed-form of $\gamma_{x \rightarrow z}$ is given by Eq.~(\ref{eq:geodesic_2pts}): 
$$\gamma_{x \rightarrow z}(t) = x \oplus_c (-x \oplus_c z) \otimes_c t$$
We denote by $x',z' \in \partial \D^n_c$ its points at infinity as described by Eq.~(\ref{eq:pts_at_infinity_geodesic}). Then, the hyperbolic angle $\angle ywx'$ is well defined from Eq.~(\ref{eq:gyro_angles}):
\begin{align}
\cos(\angle (\gamma_{w \rightarrow y}, \gamma_{x \rightarrow z})) = \cos(\angle ywz') = \frac{\inp{-w \oplus_c y}{-w \oplus_c z'}}{\|-w \oplus_c y\| \cdot \|-w \oplus_c z'\|}.
\label{eq:angle_thm_proof}
\end{align}
We now perform 2 steps for this proof.
\par \textit{i) Existence of $w$}:

The angle function from Eq.~(\ref{eq:angle_thm_proof}) is continuous w.r.t $t$ when $w = \gamma_{x \rightarrow z}(t)$. So we first prove existence of an angle of $\pi/2$ by continuously moving $w$ from $x'$ to $z'$ when $t$ goes from $-\infty$ to $\infty$, and observing that $\cos(\angle ywz')$ goes from $-1$ to $1$ as follows:
\begin{align}
\cos(\angle yx'z') = 1 \quad \& \quad \lim_{w \rightarrow z'} \cos(\angle ywz') = -1.
\label{eq:two_angles_at_inf}
\end{align}
The left part of Eq.~(\ref{eq:two_angles_at_inf}) follows from Eq.~(\ref{eq:angle_thm_proof}) and from the fact (easy to show from the definition of $\oplus_c$) that $a \oplus_c b = a$, when $\|a\| = 1/\sqrt{c}$ (which is the case of $x'$). The right part of Eq.~(\ref{eq:two_angles_at_inf}) follows from the fact that $\angle ywz' = \pi - \angle ywx'$ (from the conformal property, or from Eq.~(\ref{eq:gyro_angles})) and $\cos(\angle yz'x') = 1$ (proved as above).

Hence $\cos(\angle ywz')$ has to pass through $0$ when going from $-1$ to $1$, which achieves the proof of existence.

\par \textit{ii) Uniqueness of $w$}:

Assume by contradiction that there are two $w$ and $w'$ on $\gamma_{x \rightarrow z}$ that form angles $\angle ywx'$ and $\angle yw'x'$ of $\pi/2$. Since $w, w', x'$ are on the same geodesic, we have
\begin{align}
\pi/2 = \angle yw'x' = \angle yw'w = \angle ywx' = \angle yw'w
\end{align}
So $\Delta yww'$ has two right angles, but in the Poincar\'e ball this is impossible. 
\end{proof}

Now, we need two more lemmas:

\begin{lemma}(Minimizing distance from point to geodesic)
\label{lemma:min_dist_pt_geodesic}
The orthogonal projection of a point to a geodesic (not passing through the point) is minimizing the distance between the point and the geodesic.
\end{lemma}

\begin{proof}
The proof is similar with the Euclidean case and it's based on hyperbolic sine law and the fact that in any right hyperbolic triangle the hypotenuse is strictly longer than any of the other sides. 
\end{proof}

\begin{lemma}(Geodesics through $p$) 
\label{lemma:geos_proj_zero}
Let $\tilde{H}_{a,p}^c$ be a Poincar\'e hyperplane. Then, for any $w \in \tilde{H}_{a,p}^c \setminus \{p\}$, all points on the geodesic $\gamma_{p \rightarrow w}$ are included in $\tilde{H}_{a,p}^c$. 
\end{lemma}

\begin{proof}
$\gamma_{p \rightarrow w}(t) = p \oplus_c (-p \oplus_c w) \otimes_c t$. Then, it is easy to check the condition in Eq.~(\ref{eq:poincare_hyp_def}):
\begin{align}
\inp{-p \oplus_c \gamma_{p \rightarrow w}(t)}{a} = \inp{(-p \oplus_c w) \otimes_c t}{a} \propto \inp{-p \oplus_c w}{a}=0.
\end{align}
\end{proof}

We now turn back to our proof. Let  $x\in\D^n_c$ be an arbitrary point and $\tilde{H}_{a,p}^c$ a Poincar\'e hyperplane. We prove that there is at least one point $w^* \in\tilde{H}_{a,p}^c$ that achieves the infimum distance
\begin{align}
d_c(x, w^*) = \inf_{w \in\tilde{H}_{a,p}^c} d_c(x,w),
\label{eq:min_dist_them_proof}
\end{align}
and, moreover, that this distance is the same as the one in the theorem's statement.

We first note that for any point $w \in\tilde{H}_{a,p}^c$, if $\angle xwp \neq \pi/2$, then $w \neq w^*$. Indeed, using Lemma~\ref{lemma:min_dist_pt_geodesic} and Lemma~\ref{lemma:geos_proj_zero}, it is obvious that the projection of $x$ to $\gamma_{p \rightarrow w}$ will give a strictly lower distance.

Thus, we only consider $w \in\tilde{H}_{a,p}^{c}$ such that $\angle xwp = \pi/2$. Applying hyperbolic sine law in the right triangle $\Delta xwp$, one gets:
\begin{align}
d_c(x,w) = (1/\sqrt{c})\sinh^{-1} \left( \sinh(\sqrt{c}\ d_c(x,p)) \cdot \sin(\angle xpw) \right).
\label{eq:min_dist_sine_law_applied}
\end{align}

One of the above quantities does not depend on $w$:

\begin{align}
\sinh(\sqrt{c}\ d_c(x,p)) = \sinh(2 \tanh^{-1} (\sqrt{c}\|-p \oplus_c x\|)) = \frac{2 \sqrt{c}\|-p \oplus_c x\|}{1 - c\|-p \oplus_c x\|^2}.
\label{eq:smthg_with_sinh}
\end{align}

The other quantity is $\sin(\angle xpw)$ which is minimized when the angle $\angle xpw$ is minimized (because $\angle xpw < \pi/2$ for the hyperbolic right triangle $\Delta xwp$), or, alternatively, when $\cos(\angle xpw)$ is maximized. But, we already have from Eq.~(\ref{eq:gyro_angles}) that:
\begin{align}
\cos(\angle xpw) = \frac{\inp{-p \oplus_c x}{-p \oplus_c w}}{\|-p \oplus_c x\| \cdot \|-p \oplus_c w\|}.
\label{eq:cos_xpw}
\end{align}
To maximize the above, the constraint on the right angle at $w$ can be dropped because $\cos(\angle xpw)$ depends only on the geodesic $\gamma_{p \rightarrow w}$ and not on $w$ itself, and because there is always an orthogonal projection from any point $x$ to any geodesic as stated by Lemma~\ref{lemma:proj_geodesic}. Thus, it remains to find the maximum of Eq.~(\ref{eq:cos_xpw}) when $w \in\tilde{H}_{a,p}^c$. Using the definition of $\tilde{H}_{a,p}^c$ from Eq.~(\ref{eq:poincare_hyp_def}), one can easily prove that
\begin{align}
\{\log^c_p(w) : w \in \tilde{H}_{a,p}^c\} = \{a\}^{\perp}.
\end{align}
Using that fact that $\log^c_p(w)/\Vert\log^c_p(w)\Vert=-p \oplus_c w/\Vert-p \oplus_c w\Vert$, we just have to find
\begin{align}
\max_{z \in \{a\}^{\perp}} \left( \frac{\inp{-p \oplus_c x}{z}}{\|-p \oplus_c x\| \cdot \|z\|} \right),
\end{align}
and we are left with a well known Euclidean problem which is equivalent to finding the minimum angle between the vector $-p \oplus_c x$ (viewed as Euclidean) and the hyperplane $\{a\}^{\perp}$. This angle is given by the Euclidean orthogonal projection whose $sin$ value is the distance from the vector's endpoint to the hyperplane divided by the vector's length:
\begin{align}
\sin(\angle xpw^*) = \frac{|\inp{-p \oplus_c x}{\frac{a}{\Vert a\Vert}}|}{\|-p \oplus_c x\|}.
\label{eq:almost_done_thm_proof}
\end{align}

It follows that a point $w^* \in\tilde{H}_{a,p}^c$ satisfying Eq.~(\ref{eq:almost_done_thm_proof}) exists (but might not be unique). Combining Eqs.~(\ref{eq:min_dist_them_proof}),(\ref{eq:min_dist_sine_law_applied}),(\ref{eq:smthg_with_sinh}) and (\ref{eq:almost_done_thm_proof}) concludes the proof.
\begin{flushright}
$\square$
\end{flushright}

\section{Derivation of the Hyperbolic GRU Update-gate}\label{sec:hyp_GRU}

In \cite{tallec2018can}, the authors recover the update/forget-gate mechanism of a GRU/LSTM by requiring that the class of neural networks given by the chosen architecture be invariant to \textit{time-warpings}. The idea is the following. 

\paragraph{Recovering the update-gate from time-warping.}
A naive RNN is given by the equation 
\begin{align}
h(t+1)=\varphi(W h(t) + U x(t) +b)
\end{align}

Let's drop the bias $b$ to simplify notations. If $h$ is seen as a differentiable function of time, then a first-order Taylor development gives $h(t+\delta t)\approx h(t)+\delta t\frac{dh}{dt}(t)$ for small $\delta t$. Combining this for $\delta t=1$ with the naive RNN equation, one gets 
\begin{equation}
\dfrac{dh}{dt}(t)=\varphi(W h(t) + U x(t))-h(t).
\end{equation}
As this is written for any $t$, one can replace it by $t\leftarrow \alpha(t)$ where $\alpha$ is a (smooth) increasing function of $t$ called the \textit{time-warping}. Denoting by $\tilde{h}(t):=h(\alpha(t))$ and $\tilde{x}(t):=x(\alpha(t))$, using the chain rule $\frac{d\tilde{h}}{dt}(t)=\frac{d\alpha}{dt}(t)\frac{dh}{dt}(\alpha(t))$, one gets 
\begin{equation}\label{eq:warping}
\dfrac{d\tilde{h}}{dt}(t)=\frac{d\alpha}{dt}(t)\varphi(W \tilde{h}(t) + U \tilde{x}(t))-\frac{d\alpha}{dt}(t)\tilde{h}(t).
\end{equation}
Removing the tildas to simplify notations, discretizing back with $\frac{dh}{dt}(t)\approx h(t+1)-h(t)$ yields 
\begin{equation}\label{eq:c_warping}
h(t+1)=\frac{d\alpha}{dt}(t)\varphi(W h(t) + U x(t))+\left(1-\frac{d\alpha}{dt}(t)\right)h(t).
\end{equation}
Requiring that our class of neural networks be invariant to time-warpings means that this class should contain RNNs defined by Eq.~(\ref{eq:c_warping}), \textit{i.e.} that $\frac{d\alpha}{dt}(t)$ can be learned. As this is a positive quantity, we can parametrize it as $z(t)=\sigma(W^{z} h(t)+U^z x(t))$, recovering the forget-gate equation:
\begin{equation}
h(t+1)=z(t)\varphi(W h(t) + U x(t))+(1-z(t))h(t).
\end{equation}

\paragraph{Adapting this idea to hyperbolic RNNs.}
The \textit{gyroderivative}~\cite{birman2001hyperbolic} of a map $h:\Re\to\D_c^n$ is defined as
\begin{equation}
\dfrac{dh}{dt}(t)=\lim_{\delta t\to 0}\dfrac{1}{\delta t}\otimes_c(-h(t)\oplus_c h(t+\delta t)).
\end{equation}
Using M\"obius scalar associativity and the left-cancellation law leads us to 
\begin{equation}\label{eq:hyp_taylor}
h(t+\delta t)\approx h(t)\oplus_c\delta t\otimes_c\frac{dh}{dt}(t),
\end{equation}
for small $\delta t$. Combining this with the equation of a simple hyperbolic RNN of Eq.~(\ref{eq:naive_hyp_rnn}) with $\delta t=1$, one gets 
\begin{equation}\label{eq:hyp_eq_diff}
\frac{dh}{dt}(t)=-h(t)\oplus_c\varphi^{\otimes_c}(W\otimes_c h(t)\oplus_c U\otimes_c x(t)).
\end{equation}
For the next step, we need the following lemma:
\begin{lemma}[Gyro-chain-rule]\label{eq:gyrochainrule}
For $\alpha:\Re\to\Re$ differentiable and $h:\Re\to\D_c^n$ with a well-defined gyro-derivative, if $\tilde{h}:=h\circ \alpha$, then we have 
\begin{equation}
\frac{d\tilde{h}}{dt}(t)=\frac{d\alpha}{dt}(t)\otimes_c\frac{dh}{dt}(\alpha(t)),
\end{equation}
where $\frac{d\alpha}{dt}(t)$ denotes the usual derivative.
\end{lemma}
\begin{proof}
\begin{align}
\frac{d\tilde{h}}{dt}(t) &=\lim_{\delta t\to 0}\dfrac{1}{\delta t}\otimes_c[-\tilde{h}(t)\oplus_c \tilde{h}(t+\delta t)]\\
&=\lim_{\delta t\to 0}\dfrac{1}{\delta t}\otimes_c[-h(\alpha(t))\oplus_c h(\alpha(t)+\delta t (\alpha'(t)+\mathcal{O}(\delta t)))]\\
&=\lim_{\delta t\to 0}\dfrac{\alpha'(t)+\mathcal{O}(\delta t)}{\delta t(\alpha'(t)+\mathcal{O}(\delta t))}\otimes_c[-h(\alpha(t))\oplus_c h(\alpha(t)+\delta t (\alpha'(t)+\mathcal{O}(\delta t)))]\\
&=\lim_{\delta t\to 0}\dfrac{\alpha'(t)}{\delta t(\alpha'(t)+\mathcal{O}(\delta t))}\otimes_c[-h(\alpha(t))\oplus_c h(\alpha(t)+\delta t (\alpha'(t)+\mathcal{O}(\delta t)))]\\
&=\lim_{u\to 0}\dfrac{\alpha'(t)}{u}\otimes_c[-h(\alpha(t))\oplus_c h(\alpha(t)+u)]\\
&=\frac{d\alpha}{dt}(t)\otimes_c\frac{dh}{dt}(\alpha(t))\quad\quad\quad\ \ \ \ \ \ \ \ \ \ \ \ \ \ \ \ \ \ \ \ \ \ \ \ \ \ \ \ \ \ \ \ \ \ \ \ \ \ \ \ \ \ \text{(M\"obius scalar associativity)}
\end{align}
where we set $u=\delta t(\alpha'(t)+\mathcal{O}(\delta t))$, with $u\to 0$ when $\delta t\to 0$, which concludes.
\end{proof}
Using lemma~\ref{eq:gyrochainrule} and Eq.~(\ref{eq:hyp_eq_diff}), with similar notations as in Eq.~(\ref{eq:warping}) we have 
\begin{equation}
\dfrac{d\tilde{h}}{dt}(t)=\frac{d\alpha}{dt}(t)\otimes_c(-\tilde{h}(t)\oplus_c\varphi^{\otimes_c}(W\otimes_c \tilde{h}(t)\oplus_c U\otimes_c \tilde{x}(t))).
\end{equation}
Finally, discretizing back with Eq.~(\ref{eq:hyp_taylor}), using the left-cancellation law and dropping the tildas yields
\begin{equation}
h(t+1)=h(t)\oplus_c\frac{d\alpha}{dt}(t)\otimes_c(-h(t)\oplus_c\varphi^{\otimes_c}(W\otimes_c h(t)\oplus_c U\otimes_c x(t))).
\end{equation}
Since $\alpha$ is a time-warping, by definition its derivative is positive and one can choose to parametrize it with an update-gate $z_t$ (a scalar) defined with a sigmoid. Generalizing this scalar scaling by the M\"obius version of the pointwise scaling $\odot$ yields the M\"obius matrix scaling $\text{diag}(z_t)\otimes_c\cdot$, leading to our proposed Eq.~(\ref{eq:gru_update}) for the hyperbolic GRU.


\section{More Experimental Investigations}\label{sec:exp_res}

The following empirical facts were observed for both hyperbolic RNNs and GRUs.

We observed that, in the hyperbolic setting, accuracy is often much higher when sentence embeddings can go close to the border (hyperbolic "infinity"), hence exploiting the hyperbolic nature of the space. Moreover, the faster the two sentence norms go to 1, the more it's likely that a good local minima was reached. See figures~\ref{fig:prfx30-gru-all-hyp} and~\ref{fig:prfx30-rnn-all-hyp}.

We often observe that test accuracy starts increasing exactly when sentence embedding norms do. However, in the hyperbolic setting, the sentence embeddings norms remain close to 0 for a few epochs, which does not happen in the Euclidean case. See figures~\ref{fig:prfx30-gru-all-hyp},~\ref{fig:prfx30-rnn-all-hyp} and ~\ref{fig:prfx30-gru-all-eucl}. This mysterious fact was also exhibited in a similar way by~\citep{nickel2017poincar} which suggests that the model first has to adjust the angular layout in the almost Euclidean vicinity of 0 before increasing norms and fully exploiting hyperbolic geometry.

\begin{figure}[h!]
  \centering
  \begin{subfigure}[b]{1\linewidth}
    \includegraphics[width=\linewidth]{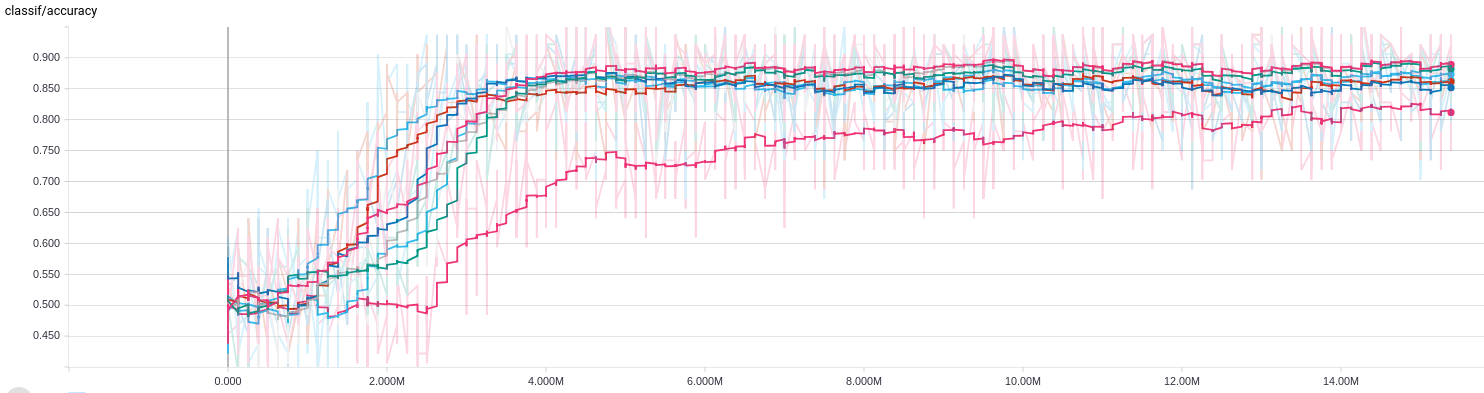}
    \caption{Test accuracy}
  \end{subfigure}\\
  \begin{subfigure}[b]{1\linewidth}
    \includegraphics[width=\linewidth]{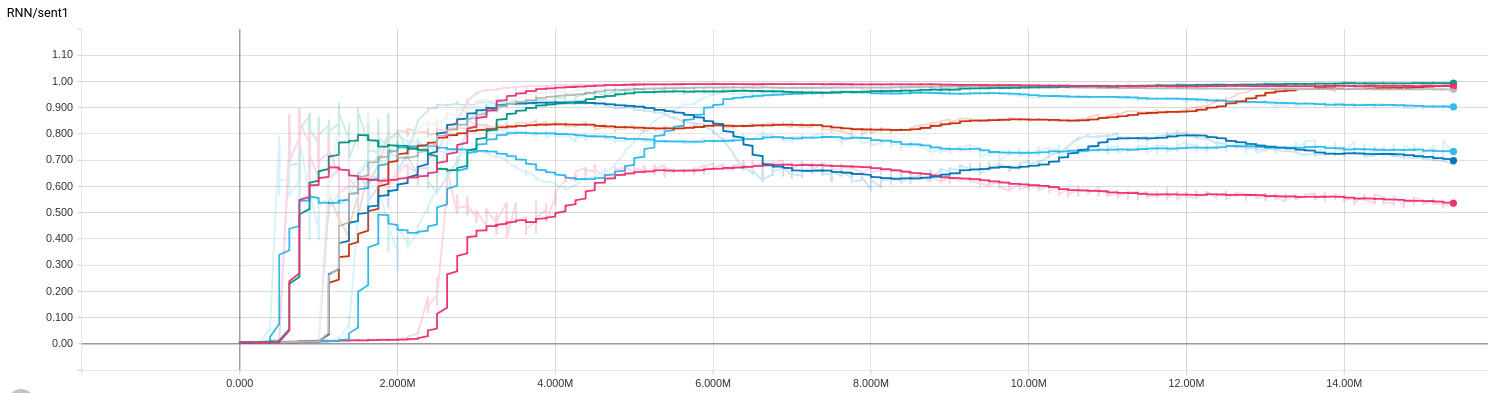}
    \caption{Norm of the first sentence. Averaged over all sentences in the test set.}
  \end{subfigure}
  \caption{PREFIX-30\% accuracy and first (premise) sentence norm plots for different runs of the same architecture: hyperbolic GRU followed by hyperbolic FFNN and hyperbolic/Euclidean (half-half) MLR. The X axis shows millions of training examples processed.}
  \label{fig:prfx30-gru-all-hyp}
\end{figure}

\begin{figure}[h!]
  \centering
  \begin{subfigure}[b]{1\linewidth}
    \includegraphics[width=\linewidth]{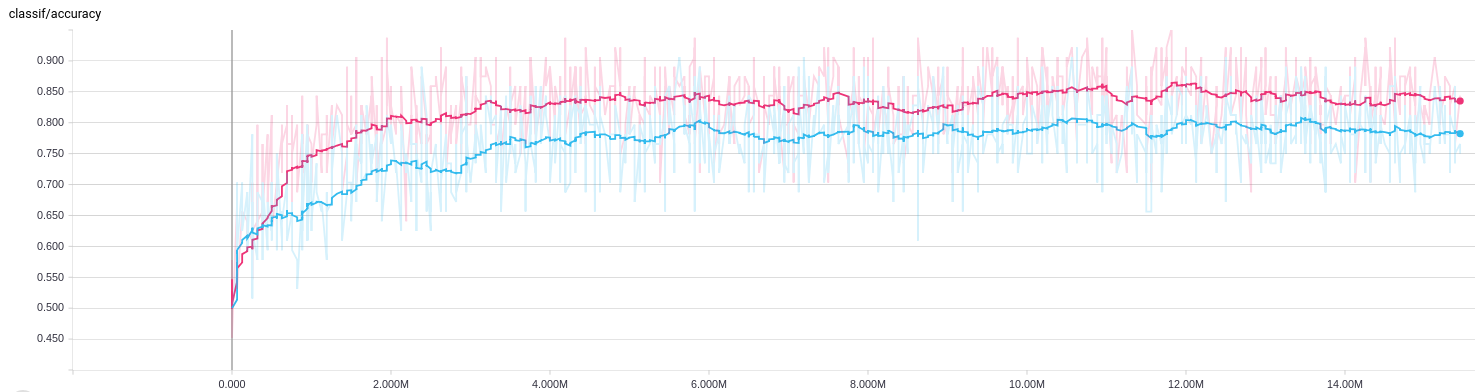}
    \caption{Test accuracy}
  \end{subfigure}\\
  \begin{subfigure}[b]{1\linewidth}
    \includegraphics[width=\linewidth]{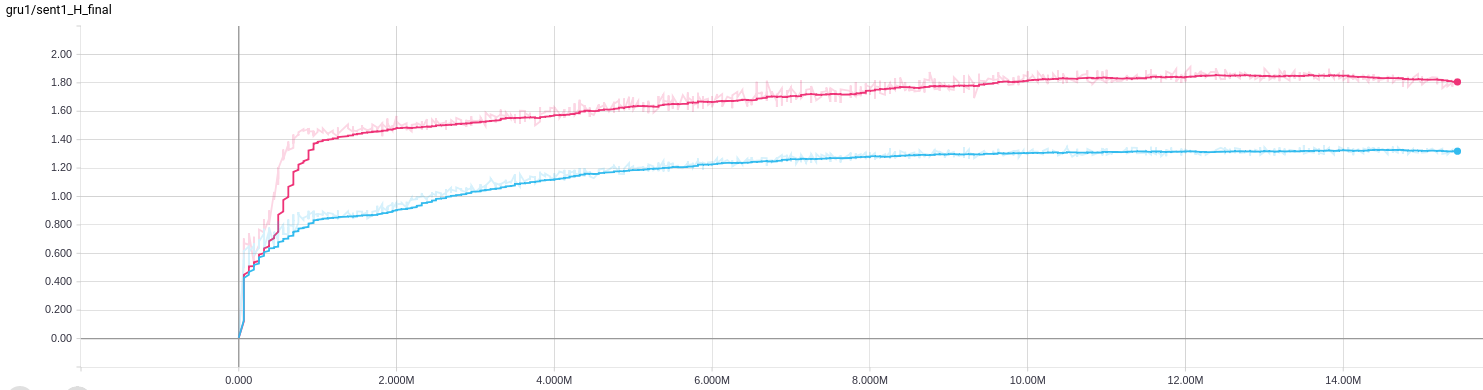}
    \caption{Norm of the first sentence. Averaged over all sentences in the test set.}
  \end{subfigure}
  \caption{PREFIX-30\% accuracy and first (premise) sentence norm plots for different runs of the same architecture: Euclidean GRU followed by Euclidean FFNN and Euclidean MLR. The X axis shows millions of training examples processed.}
  \label{fig:prfx30-gru-all-eucl}
\end{figure}

\begin{figure}[h!]
  \centering
  \begin{subfigure}[b]{1\linewidth}
    \includegraphics[width=\linewidth]{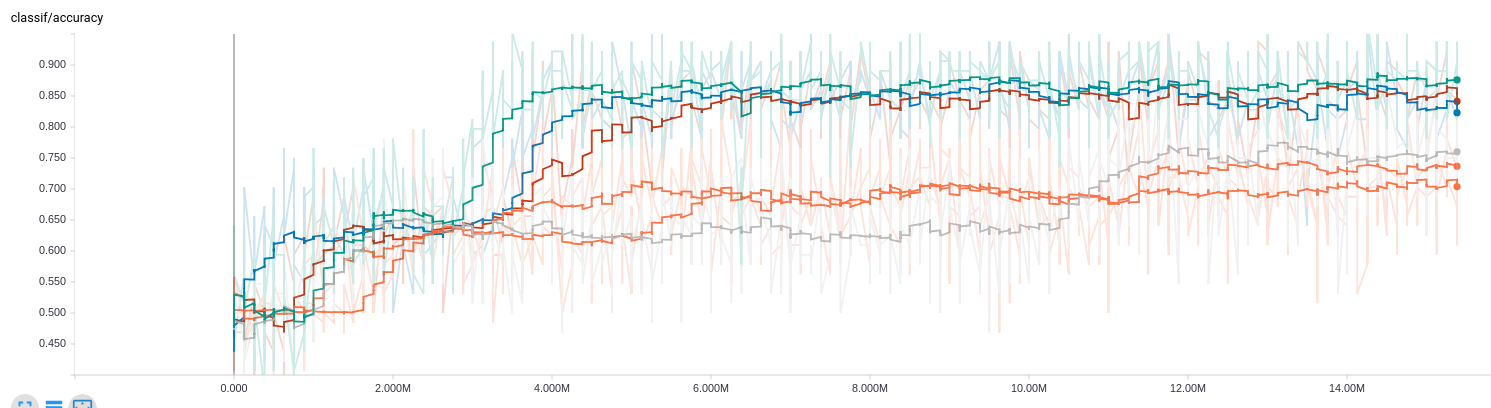}
    \caption{Test accuracy}
  \end{subfigure}\\
  \begin{subfigure}[b]{1\linewidth}
    \includegraphics[width=\linewidth]{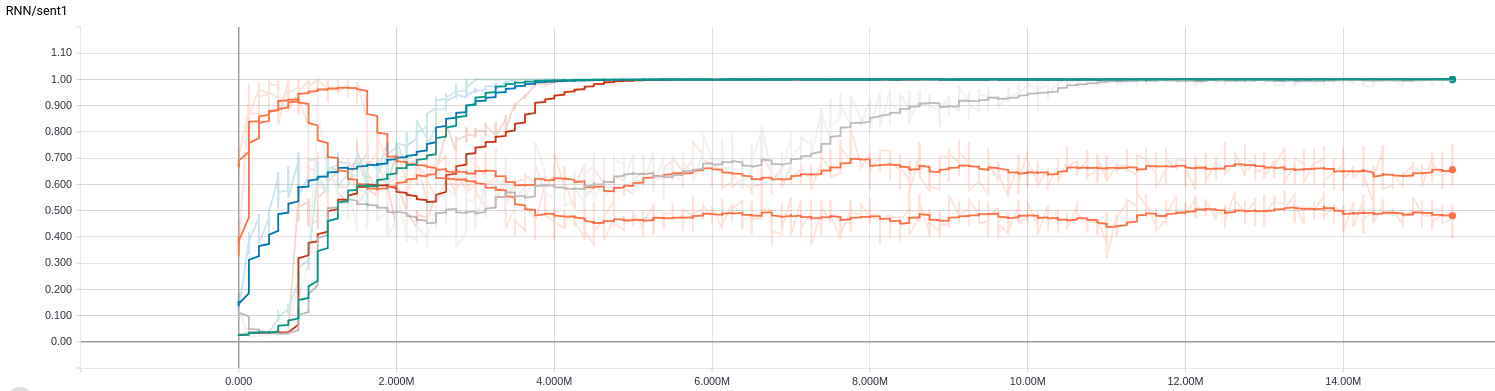}
    \caption{Norm of the first sentence. Averaged over all sentences in the test set.}
  \end{subfigure}
  \caption{PREFIX-30\% accuracy and first (premise) sentence norm plots for different runs of the same architecture: hyperbolic RNN followed by hyperbolic FFNN and hyperbolic MLR. The X axis shows millions of training examples processed.}
  \label{fig:prfx30-rnn-all-hyp}
\end{figure}

\end{document}